\documentclass[letterpaper, journal]{IEEEtran}

\usepackage{amsmath,amsthm,graphicx,amssymb,fullpage,wrapfig}

\usepackage{setspace}

\usepackage{amsmath,amsthm,amssymb,algorithmic,algorithm,hyperref}
\usepackage[margin=1in]{geometry}
\usepackage{setspace}
\usepackage{color}
\usepackage{soul}
\usepackage{epstopdf}
\usepackage{cite}

\theoremstyle{remark}

\hypersetup{colorlinks=true}

\newcommand{\ben}{\begin{eqnarray}}

\newcommand{\een}{\end{eqnarray}}

\newcommand{\transpose}{^{\intercal}}
\newtheorem{thm}{Theorem}
\newtheorem{corollary}{Corollary}
\newtheorem{lemma}{Lemma}

\newcommand{\ka}{\kappa}
\newcommand{\M}{{M}}

\newcommand{\e}{{e}}

\def\yc{\textcolor{black}}
\def\yang{\textcolor{black}}

\newtheorem{remark}{Remark}

\title{Sketching for Sequential Change-Point Detection}
\author{Yang Cao\thanks{Yang Cao
    (Email: caoyang@gatech.edu) and Yao Xie (Email: yao.xie@isye.gatech.edu)
   are with the H. Milton Stewart School of
    Industrial and Systems Engineering, Georgia Institute of
    Technology, Atlanta, GA, USA. Andrew Thompson (Email: thompson@maths.ox.ac.uk) is with the Mathematical Institute, University of Oxford, Oxford, UK. Meng Wang (Email: wangm7@rpi.edu) is with the Department of Electrical, Computer \& Systems Engineering, Rensselaer Polytechnic Institute, Troy, NY, USA.
    Authors are listed alphabetically. This paper was presented [in part] at the  GlobalSIP 2015 \cite{xie2015sketching}.}, Andrew Thompson, Meng Wang, and Yao Xie
}
\begin{document}
%\ninept
%
\maketitle
\begin{abstract}
We study sequential change-point detection procedures based on linear sketches of high-dimensional signal vectors using generalized likelihood ratio (GLR) statistics. The GLR statistics allow for an unknown post-change mean that represents an anomaly or novelty. We consider both fixed and time-varying projections, derive theoretical approximations to two fundamental performance metrics: the average run length (ARL) and the expected detection delay (EDD); these approximations are shown to be highly accurate by numerical simulations. \yang{We further characterize the relative performance measure of the sketching procedure compared to that without sketching and show that there can be little performance loss when the signal strength is sufficiently large, and enough number of sketches are used.} Finally, we demonstrate the good performance of sketching procedures using simulation and real-data examples on solar flare detection and failure detection in power networks.
\end{abstract}
%
%\begin{keywords}
%One, two, three, four, five
%\end{keywords}
%%

\section{Introduction}

Online change-point detection from high-dimensional streaming data is a fundamental problem arising from applications such as real-time monitoring of sensor networks, computer network anomaly detection and computer vision (e.g., \cite{changepoint_new_book2014, PoorHadjiliadis2008}). To reduce data dimensionality, a conventional approach is \emph{sketching} (see, e.g., \cite{sketchingWoodruff2014}), which performs random projection of the high-dimensional data vectors into lower-dimensional ones. \yang{Sketching has now been widely used in signal processing and machine learning to reduce dimensionality and algorithm complexity, and achieve various practical benefits \cite{dasarathy2012covariance, chi2016kronecker, wang2015fast, alaoui2015fast, bachrach2009sketching, raskutti2015statistical, indyk2008explicit}.}

We consider change-point detection using linear sketches of high-dimensional data vectors. \yang{Sketching reduces the computational complexity of the detection statistic from $\mathcal O(N)$ to $\mathcal O(M)$, where $N$ is the original dimensionality and $M$ is the dimensionality of sketches. Since we would like to perform real-time detection, any reduction in computational complexity (without incurring much performance loss) is highly desirable. Sketching also offers  practical benefits. For instance, for large sensor networks, it reduces the burden of data collection and transmission. It may be impossible to collect data from all sensors and transmit them to a central hub in real-time, but this can be done if we only select a small subset of sensors to collect data at each time. Sketching also reduces data storage requirement. For instance, change-point detection using the generalized likelihood ratio statistic, although robust, it is non-recursive, and one has to store historical data. Using sketching, we only need to store the much lower dimensional sketches rather than the original high-dimensional vectors.}

In this paper, we present a new sequential \emph{sketching procedure} based on the generalized likelihood ratio (GLR) statistics.  In particular, suppose we may choose an $M\times N$ matrix $A$ with $M\ll N$ to project the original data: $y_t = Ax_t$, $t = 1, 2, \ldots$. Assume the pre-change vector is zero-mean Gaussian distributed and the post-change vector is Gaussian distributed with an {\it unknown} mean vector $\mu$ while the covariance matrix is unchanged. Here we assume the mean vector is unknown since it typically represents an anomaly. The GLR statistic is formed by replacing the unknown $\mu$ with its maximum likelihood ratio estimator (e.g., \cite{siegmund2011detecting}). Then we further generalize to the setting with time-varying projections $A_t$ of dimension $M_t\times N$. We demonstrate the good performance of our procedures by simulations,  a real-data example of solar flare detection, and a synthetic example of power network failure detection with data generated using real-world power network topology. %\yang{Our results show that the sketching procedure can achieve good performance relative to the procedure using full data with  much lower dimensional sketches than $N$ when the signal strength is sufficiently strong. }

Our theoretical contribution are mainly in two aspects. We obtain analytic expressions for two fundamental performance metrics for the sketching procedures: the average run length (ARL) when there is no change and the expected detection delay (EDD) when there is a change-point, for both fixed and time-varying projections. Our approximations are shown to be highly accurate using simulations. These approximations are quite useful in determining the threshold of the detection procedure to control false-alarms, without having to resort to the onerous numerical simulations. Moreover, we characterize \yang{the relative performance of the sketching procedure compared to that without sketching.} 
We examine the EDD ratio when the sketching matrix $A$ is either a random Gaussian matrix or a sparse 0-1 matrix (in particular, an expander graph). We find that, as also verified numerically, \yang{when the signal strength  and $M$ are sufficiently large, the sketching procedure may have  little performance loss. When the signal is weak, the performance loss can be large when $M$ is too small. In this case, our results can be used to find the minimum $M$ such that performance loss is bounded, assuming certain worst case signal and for a given target ARL value. }

\yang{To the best of our knowledge, our work is the first to consider sequential change-point detection using the generalized likelihood ratio statistic, assuming the {\it post-change mean is unknown to represent an anomaly}. The only other work \cite{compressive_change2014} that consider change-point detection using linear projections assume the post-change mean is known and further, to be sparse. Our results are more general since we do not make such assumptions. Assuming the post-change mean to be unknown provides a more useful procedure since in change-point detection, the post-change set-up is usually unknown. Moreover, \cite{compressive_change2014} considers Shiryaev-Robert's procedure, which is based on a different kind of detection statistic than the generalized likelihood ratio statistic considered here. The theoretical analyses therein consider slightly different performance measures, the Probability of False Alarm (PFA) and Average Detection Delay (ADD) and our analyses are completely different. }

Our notations are standard: $\chi^2_k$ denotes the Chi-square distribution with degree-of-freedom $k$, $I_n$ denotes an identity matrix of size $n$; $X^\dag$ denotes the pseudoinverse of a matrix $X$; $[x]_i$ denotes the $i$th coordinate of a vector $x$; $[X]_{ij}$ denotes the $ij$th element of a matrix $X$; $x\transpose$ denotes the transpose of a vector or matrix $x$.

The rest of the sections are organized as follows. We first review some  related work. Section \ref{sec:formulation} sets up the formulation of the sketching problem for sequential change-point detection. Section \ref{sec:sketching} presents the sketching procedure. Section \ref{sec:analysis} and Section \ref{sec:A} contain the performance analysis of the sketching procedures. Section \ref{sec:numerical} and Section \ref{sec:real} demonstrate good performance of our sketching procedures using simulation and real-world examples. Section \ref{sec:summary} concludes the paper. All proofs are delegated to the appendix.

\subsection{Related work}

Change-point detection problems are closely related to industrial quality control and multivariate statistical control charts (SPC), where an observed process is assumed initially to be in control and at a change-point becomes out of control. The idea of using random projections for change detection has been explored for SPC in the pioneering work \cite{Runger96} based on $U^2$ multivariate control chart, the follow-up work \cite{Bodnar05} for cumulative sum (CUSUM) control chart and the exponential weighting moving average (EWMA) schemes, and in \cite{Skubalska2013,change-point-bookchapt15} based on the Hotelling statistic. These works provide a complementary perspective from SPC design, while our method takes a different approach and is based on sequential hypothesis testing, treating both the change-point location and the post-change mean vector as unknown parameters. By treating the change-point location as an unknown parameter when deriving the detection statistic, the sequential hypothesis testing approach overcomes the drawback of some SPC methods due to a lack-of-memory, such as the Shewhart chart and the Hotelling chart, since they cannot utilize the information embedded in the entire sequence \cite{SQC06}. Moreover, our sequential GLR statistic may be preferred over the CUSUM procedure in the setting when it is difficult to specify the post-change mean vector.
Besides the above distinctions from the SPC methods, other novelty of our methods also include: (1) we developed new theoretical results for the sequential GLR statistic; (2) we consider the sparse 0-1 and time-varying projections (the sparse 0-1 projection corresponds to downsampling the dimensions); (3) we study the amount of dimensionality reduction can be performed (i.e., the minimum $M$) such that there is little performance loss.

This paper extends on our preliminary work reported in \cite{xie2015sketching} with several important extensions. We have added (1) time-varying sketching projections and their theoretical  analysis; (2) extensive numerical examples to verify our theoretical results; (3) new real-data examples of solar flare detection and power failure detection.

Our work is related to compressive signal processing \cite{compressiveSP2010}, where the problem of interest is to estimate or detect (in the fixed-sample setting) a sparse signal using compressive measurements. In \cite{arias2012detecting}, an offline test for a non-zero vector buried in Gaussian noise using linear measurements is studied; interestingly, a conclusion similar to ours is drawn that the task of detection within this setting is much easier than the tasks of estimation and support recovery. Another related work is \cite{GengXuLai2013}, which considers a problem of identifying a subset of data streams within a larger set, where the data streams in the subset follow a distribution (representing anomaly) that is different from the original distribution; the problem considered therein is not a sequential change-point detection problem as the ``change-point''  happens at the onset ($t = 1$). 
In \cite{XuLai2013}, an offline setting is considered and the goal is to identify $k$ out of $n$ samples whose distributions are different from the normal distribution $f_0$. They use a ``temporal'' mixing of the samples over the finite time horizon. This is different from our setting since we project over the signal dimension at each time.  Other related work include kernel methods \cite{harchaoui2009kernel} and \cite{arlot2012kernel} that focus on offline change-point detection. Finally,  detecting transient changes in power systems has been studied in \cite{chen2016quickest}.

Another common approach to dimensionality reduction is principal component analysis (PCA) \cite{mishin2014real}, which achieves dimensionality reduction  by projecting the signal along the singular space of the leading singular values. In this case, $A$ or $A_t$ corresponds to the signal singular space. Our theoretical approximation for ARL and EDD can also be applied in these settings. \yc{It may not be easy to find the signal singular space when the dimensionality is high, since computing singular value decomposition can be expensive \cite{PCA_changepoint}.}

\section{Formulation}
\label{sec:formulation}

Consider a sequence of observations with an open time horizon $x_1, x_2, \ldots, x_t$, $t = 1, 2, \ldots$, where $x_t \in \mathbb{R}^N$ and $N$ is the signal dimension. Initially, the observations are due to noise. There can be a time $\kappa$ such that an unknown change-point occurs and it changes the mean of the signal vector. Such a problem can be formulated as the following hypothesis test:
\begin{equation}
\begin{array}{ll}
\textsf{H}_0: & x_t  \sim \mathcal{N}(0, I_N), \quad t = 1, 2, \ldots\\
\textsf{H}_1: & x_t \sim \mathcal{N}(0, I_N), \quad t = 1, 2, \ldots, \ka, \\
& x_t \sim \mathcal{N}(\mu, I_N), \quad t = \ka + 1, \ka+2, \ldots
\end{array}
\label{hypothese}
\end{equation}
where the {\it unknown} mean vector is defined as
\[\mu \triangleq [\mu_1, \ldots, \mu_N]^\intercal \in \mathbb{R}^N.\]
Without loss of generality, we have assumed the noise variance is 1. %
Our goal is to detect the change-point as soon as possible after it occurs, subjecting to the false alarm constraint. Here, we assume the covariance of the data to be an identity matrix and the change only happens to the mean. 

To reduce data dimensionality, we  linearly project each observation $x_t$ into a lower dimensional space, which we refer to as \emph{sketching}. We aim to develop procedures that can detect a change-point using the low-dimensional \emph{sketches}. In the following, we consider two types of linear sketching: the fixed projection and the time-varying projection.

\vspace{0.05in}

\noindent{\it Fixed projection.} Choose an $M\times N$ projection matrix $A$ with $M\ll N$. We obtain low dimensional sketches via:
\begin{equation}
y_t\triangleq Ax_t, \quad t = 1, 2, \ldots
\label{sketching}
\end{equation}
Then the hypothesis test for the original problem (\ref{hypothese}), becomes the following hypothesis test based on the sketches (\ref{sketching})
\begin{equation}
\begin{array}{ll}
\textsf{H}_0: & y_t  \sim \mathcal{N}(0, AA\transpose), \quad t = 1, 2, \ldots\\
\textsf{H}_1: & y_t \sim \mathcal{N}(0, AA\transpose), \quad t = 1, 2, \ldots, \ka, \\
& y_t \sim \mathcal{N}(A\mu, AA\transpose), \quad t = \ka + 1, \ka+2, \ldots
\end{array}\label{new_hypothesis}
\end{equation}
Note that both mean and covariance structures are affected by the projections.

\vspace{0.05in}

\noindent{\it Time-varying projection.} In certain applications, one may use different sketching matrices at each time. The projections are denoted by $A_t \in \mathbb{R}^{M_t\times N}$ and  the number of sketches $M_t$ can change as well. The hypothesis test for sketches becomes:
\begin{equation}
\begin{array}{ll}
\textsf{H}_0: & y_t  \sim \mathcal{N}(0, A_tA_t\transpose), \quad t = 1, 2, \ldots\\
\textsf{H}_1: & y_t \sim \mathcal{N}(0, A_tA_t\transpose), \quad t = 1, 2, \ldots, \ka, \\
& y_t \sim \mathcal{N}(A_t\mu, A_tA_t\transpose), \quad t = \ka + 1, \ka+2, \ldots
\end{array}
\end{equation}
The above models also capture several cases:\\
\noindent (i) (Pairwise comparison.)
In applications such as social network data analysis and computer vision, we are interested in a pairwise comparison of variables \cite{chen2015information,massiminoone}. \yang{This can be modeled as observing the difference between a pair of variables, i.e., at each time $t$, the measurements are $[x_t]_i -[x_t]_j$, for a set of $i\neq j$. There are a total of $N^2$ possible comparisons, and we may select $M$ out of $N^2$ such comparisons to observe.} The pairwise comparison model leads to a structured fixed projection $A$ with only $\{0, 1, -1\}$ entries, for instance:
\[
A =
\begin{bmatrix}
    1       & -1 & 0 & \dots & 0 & 0 \\
    1       & 0 & 0 & \dots & -1 & 0\\
    \hdotsfor{6} \\
    0       & 1 & 0 & \dots & 0& -1
\end{bmatrix} \in \mathbb{R}^{M\times N}.
\]

\noindent (ii) (Missing data.) In various applications we may only observe a subset of entries at each time (e.g., due to sensor failure), and the locations of the observed entries also vary with time \cite{balzano2015local}. This corresponds to $A_t \in \mathbb{R}^{M_t \times N}$ being a submatrix of an identity matrix by selecting rows from an index set $\Omega_t$ at time $t$.

\noindent (iii) (PCA.) There are also approaches to change-point detection using principal component analysis (PCA) of the data streams (e.g., \cite{mishin2014real, Jackson91}), which can be viewed as using a fixed projection $A$ being the signal singular space associated with the leading singular values.
%\end{enumerate}

\section{Sketching procedure}
\label{sec:sketching}

\subsection{Fixed projection}

\subsubsection{Derivation of GLR statistic}
We now derive the likelihood ratio statistic for the hypothesis test in (\ref{new_hypothesis}). 
Define the sample mean within a window $[k, t]$
\ben
\bar{y}_{k, t} =\frac{\sum_{i=k+1}^t y_i}{t-k}.
\label{ykt}
\een
Since the observations are i.i.d. over time, for an assumed change-point $\ka = k$, \yc{for the hypothesis test} in (\ref{new_hypothesis}), the log-likelihood ratio (log-LR) of observations accumulated up to time $t > k$ can be shown to be 
\yang{
\begin{equation}
\begin{split}
&\ell(t, k, \mu) \\%\nonumber = \sum_{i=k+1}^t [y_i\transpose(AA\transpose)^{-1}A\mu - \frac{1}{2} \mu\transpose A\transpose(AA\transpose)^{-1}A\mu] 
= & ~\log \frac{\prod_{i=1}^k f_0(y_i) \cdot \prod_{i=k+1}^t f_1(y_i)}{\prod_{i=1}^t f_0(y_i)} \\
= &~  \sum_{i=k+1}^t \log \frac{f_1(y_i)}{f_0(y_i)} \\
= &~(t-k)[\bar{y}_{k, t} \transpose(AA\transpose)^{-1}A\mu - \frac{1}{2} \mu\transpose A\transpose(AA\transpose)^{-1}A\mu],
\label{stats}
\end{split}
\end{equation}}
\yang{where $f_0(y_i)$ denotes the probability density function of $y_i$ under the null, i.e., $f_0 =  \mathcal{N}(0, AA\transpose)$ and $f_1(y_i)$ denotes the probability density function of $y_i$ under the alternative, i.e., $f_1 = \mathcal{N}(A\mu, AA\transpose)$. }

Since $\mu$ is unknown, we replace it with a maximum likelihood estimator for  fixed values of $k$ and $t$ in the likelihood ratio (\ref{stats}) to obtain the log-generalized likelihood ratio (log-GLR) statistic. Taking the derivative of $\ell(t, k, \mu)$ in (\ref{stats}) with respect to $\mu$ and setting it to zero, we obtain an equation that the maximum likelihood estimator $\mu^*$ of the post-change mean vector needs to satisfy: 
\begin{equation}
 A\transpose(AA\transpose)^{-1}A\mu^*=A\transpose(AA\transpose)^{-1}\bar{y}_{t, k},\label{normal_eqn}
\end{equation}
or equivalently
\begin{equation*}
 A\transpose[(AA\transpose)^{-1}A\mu^*-(AA\transpose)^{-1}\bar{y}_{t, k}]=0.
\end{equation*}
Note that since $A^\intercal$ is of dimension $M$-by-$N$, this defines an under-determined system of equations for the maximum likelihood estimator $\mu^*$.   
In other words, any $\mu^*$ that satisfies 
\[(AA\transpose)^{-1}A\mu^* = (AA\transpose)^{-1}\bar{y}_{t, k} + c,\]
for a vector $c\in \mathbb{R}^N$ that lies in the null space of $A$, $A\transpose c = 0$, is a maximum likelihood estimator for the post-change mean. In particular, we may choose the zero vector $c = 0$, then the corresponding maximum estimator  satisfies the equation below: 
\begin{equation}
(AA\transpose)^{-1}A\mu^* = (AA\transpose)^{-1}\bar{y}_{t, k}. \label{mu_star}
\end{equation}
Substituting such a $\mu^*$ into (\ref{stats}), we form the log-GLR. Using (\ref{mu_star}), the first and second terms in (\ref{stats}) become, respectively,
\begin{align*}
\bar{y}_{k, t} \transpose(AA\transpose)^{-1}A\mu^*
&= \bar{y}_{k, t} \transpose  (AA\transpose)^{-1}\bar{y}_{t, k}, \\
\frac{1}{2} {\mu^*}\transpose A\transpose(AA\transpose)^{-1}A\mu^* &=
%\frac{1}{2} {\mu^*}\transpose A\transpose(AA\transpose)^{-1}\bar{y}_{t,k}
%=
 \frac{1}{2} \bar{y}_{k,t} \transpose(AA\transpose)^{-1} \bar{y}_{t,k}.
 \end{align*}
Combining above, from (\ref{stats}) we have that the log-GLR statistic is given by
\begin{equation}
\begin{split}
%&(t-k)[\bar{y}_{k,t} \transpose(AA\transpose)^{-1}A\mu^* - \frac{1}{2} \bar{y}_{k,t} \transpose(AA\transpose)^{-1}A\mu^*] \\
%= &
\ell(t, k, \mu^*) = \frac{t-k}{2}\bar{y}_{k,t} \transpose (AA\transpose)^{-1}\bar{y}_{k,t}.
\end{split}
\label{12}
\end{equation}
%which is the log generalized likelihood ratio (GLR) statistics.
Since the change-point location $k$ is unknown, when forming the detection statistic, we take the maximum over a set of possible locations, i.e., the most recent samples from $t-w$ to $t$, where $w>0$ is the window size. 
Then we  define the {\it sketching procedure}, \yang{which is a stopping time} that stops  whenever the log-GLR statistic raises above a threshold $b > 0$:
\ben
T = \inf\{t:
\max_{t-w\leq k < t}\frac{t-k}{2}\bar{y}_{k,t} \transpose (AA\transpose)^{-1}\bar{y}_{k,t} > b\}. \label{proc}
\een
Here the role of the window is two-fold: it reduces the data storage when implement the detection procedure, and it establishes a minimum level of change that we want to detect.

\subsubsection{Equivalent formulation of fixed projection sketching procedure}
We can further simplify the log-GLR statistic in (\ref{12}) using the singular value decomposition (SVD) of $A$. This will facilitates the performance analysis in Section \ref{sec:analysis} and leads into some insights about the structure of the log-GLR statistic. Let the SVD of $A$ be given by
\begin{equation}
A = U\Sigma V\transpose, \label{A_SVD}
\end{equation}
where $U\in \mathbb{R}^{M\times M}$, $V\in \mathbb{R}^{N\times M}$ are the left and right singular spaces, $\Sigma \in \mathbb{R}^{M\times M}$ is a diagonal matrix containing all non-zero singular values. Then $(AA^\intercal)^{-1} =  U\Sigma^{-2} U^\intercal$. Thus, we can write the log-GLR statistic (\ref{12}) as
\ben
\ell(t, k, \mu^*) = \frac{t-k}{2}\bar{y}_{k,t}\transpose U\Sigma^{-2}U\transpose \bar{y}_{k,t}.
\label{final3}\een
Substitution of the sample average (\ref{ykt}) into (\ref{final3}) results in 
\[\ell(t, k, \mu^*) =  \frac{\left\|\Sigma^{-1}U\transpose \left(\sum_{i=k+1}^t y_i\right)\right\|^2}{2(t-k)}.\] Now define transformed data 
\[z_i \triangleq \Sigma^{-1} U\transpose y_i.\] Since under the null hypothesis $y_i \sim \mathcal{N}(0, AA\transpose)$, we have $z_i \sim \mathcal{N}(0, I_M)$. Similarly, under the alternative hypothesis $y_i \sim \mathcal{N}(A\mu, AA\transpose)$, we have $z_i \sim \mathcal{N}(V\transpose \mu, I_M)$. 
Combing above, we obtain the following equivalent form for the sketching procedure in (\ref{proc}):
\ben
T' = \inf\{t:
\max_{t-w\leq k < t} \frac{\left\|\sum_{i=k+1}^t z_i\right\|^2}{2(t-k)}> b\}. \label{proc2}
\een
This form of the detection procedure has one intuitive explanation: the sketching procedure essentially projects the data to form $M$ (less than $N$) {\it independent} data streams, and then form a log-GLR statistic for these independent data streams.

\subsection{Time-varying projection}
\subsubsection{GLR statistic}

Similarly, we can derive the \yc{log-LR} statistic for the time-varying projections. For an assumed change-point $\ka = k$, using all observations from $k+1$ to time $t$, we find the log likelihood ratio statistic similar to (\ref{stats}):
\begin{equation}
\begin{split}
&\ell(t, k, \mu) \\
=& \sum_{i=k+1}^t [y_i\transpose(A_iA_i\transpose)^{-1}A_i\mu - \frac{1}{2} \mu\transpose A_i\transpose(A_iA_i\transpose)^{-1}A_i\mu].
\end{split}
\label{stats2}
\end{equation}
Similarly, we replace the unknown post-change mean vector $\mu$ by its maximum likelihood estimator using data in $[k+1, t]$. 
Taking the derivative of $\ell(t, k, \mu)$ in (\ref{stats2}) with respect to $\mu$ and setting it to zero, we obtain an equation that the maximum likelihood estimator $\mu^*$ needs to satisfy
\begin{equation}
\left[\sum_{i=k+1}^t A_i\transpose(A_iA_i\transpose)^{-1}A_i\right]\mu^*=\sum_{i=k+1}^t A_i\transpose(A_iA_i\transpose)^{-1}y_i.\label{normal_eqn2}
\end{equation}

To solve $\mu^*$ from the above, one needs to discuss the rank of the matrix $\sum_{i=k+1}^t A_i\transpose(A_iA_i\transpose)^{-1}A_i$  on the left-hand-side of (\ref{normal_eqn2}). Define the SVD of $A_i = U_i D_i V_i^{\transpose}$ with $U_i \in \mathbb{R}^{M_i \times \M_i}$ and $V_i \in \mathbb{R}^{N \times M_i}$ being the eigenspaces, and $D_i \in \mathbb{R}^{M_i \times \M_i}$ being a diagonal matrix that contains all the singular values. We have that
\begin{equation}
\sum_{i=k+1}^t A_i\transpose(A_iA_i\transpose)^{-1}A_i = \sum_{i=k+1}^t V_i V_i^{\transpose} = QQ^{\transpose},
\label{V_def}
\end{equation}
where $Q=[V_{k+1}, \ldots, V_t] \in \mathbb{R}^{N \times S}$ and $S = \sum_{i=k+1}^t M_i$. Consider the rank of $\sum_{i=k+1}^t A_i\transpose(A_iA_i\transpose)^{-1}A_i$ for two cases below:
\begin{enumerate}
\item[(i)]
 $A_i'$s are independent Gaussian random matrices.  The columns within each $V_i$ in (\ref{V_def}) are linearly independent with probability one. Moreover, the columns in different $V_i$ blocks are independent since $V_i'$s are independent, and their entries are drawn as independent continuous random variables. Therefore, the columns of $Q$ are linearly independent and rank($QQ^{\transpose}$) = $\min\{N, (t-k)M\}$ with probability one. Hence, we have that if $t-k<N/M$, $QQ^{\transpose}$ is rank deficient with probability 1; if $t-k\geq N/M$, $QQ^{\transpose}$ is full rank with probability one; 
 
\item[(ii)] $A_i'$s are independent random matrices with entries drawn from a discrete distribution. In this case, we can claim that if $t-k<N/M$, $QQ^{\transpose}$ is not full rank and if $t-k\geq N/M$, $QQ^{\transpose}$ is full rank with high probability, however, with probability less than one.
\end{enumerate}
   %Since rank of $A_i\transpose(A_iA_i\transpose)^{-1}A_i \in \mathbb{R}^{M\times N}$ is less than or equal to $M$ for any $i\in [k+1,t]$, the matrix $\sum_{i=k+1}^t A_i\transpose(A_iA_i\transpose)^{-1}A_i$ is not full rank and thus is not invertible whenever $t-k < N/M$. If

From above we can see that this matrix is rank deficient when $t-k$ is small. However, this is generally the case since we want to detect quickly. 
Since the matrix in (\ref{V_def}) is non-invertible in general, we use the pseudo-inverse of the matrix. Define
\[B_{k,t} \triangleq \left(\sum_{i=k+1}^t A_i\transpose(A_iA_i\transpose)^{-1}A_i\right)^{\dag} \in \mathbb{R}^{N\times N}.\] From (\ref{normal_eqn2}), we obtain an estimate of the maximum likelihood estimator for the post-change mean
\[
\mu^* = B_{k,t}\sum_{i=k+1}^t A_i\transpose(A_iA_i\transpose)^{-1}y_i.
\]
Substituting such a $\mu^*$ into (\ref{stats2}), we obtain the \yc{log-GLR statistic} for time-varying projection:
\begin{equation}
\begin{split}
&\ell(t, k, \mu^*) \\
=&\left(\sum_{i=k+1}^t A_i\transpose(A_iA_i\transpose)^{-1}y_i\right)\transpose B_{k,t} \left(\sum_{i=k+1}^t A_i\transpose(A_iA_i\transpose)^{-1}y_i\right).
\label{GLR2}
\end{split}
\end{equation}
%We define a procedure that detects a change whenever the maximum of log-GLR statistic $\max_{t-w\leq k<t}\ell(t, k, \mu^*)$ exceeds a pre-specified threshold. %In the general case, computation of the pseudo-inverse may occur a high cost on the order of $\mathcal{O}(N^3)$.
%

\subsubsection{Time-varying 0-1 project matrices}
To further simplify the expression of GLR in (\ref{GLR2}), we focus on a special case when $A_t$ has $0$-$1$ entries only. This means that at each time we only observe a subset of the entries. 

In this case, $A_t A_t \transpose$ is an $M_t$-by-$M_t$ identity matrix, and $A_t \transpose A_t$ is a diagonal matrix. For a diagonal matrix $D\in \mathbb{R}^{N\times N}$ with diagonal entries $\lambda_1,\ldots, \lambda_N$, the pseudo-inverse of $D$ is a diagonal matrix with diagonal entries $\lambda_i^{-1}$ if $\lambda_i \neq 0$ and with diagonal entries $0$ if $\lambda_i =0$. Let the index set of the observed entries at time $t$ be $\Omega_t$. 
Define indicator variables 
\begin{equation}
\mathbb{I}_{tn} = \left\{
\begin{array}{ll}
1, & \hbox{if } n \in \Omega_t; \\
0, & \hbox{otherwise.}
\end{array}\right.
\end{equation}
Then the log-GLR statistic in (\ref{GLR2}) becomes
\begin{equation}
\ell(t, k, \mu^*) = \sum_{n=1}^N \frac{\left(\sum_{i=k+1}^t [x_i]_n \mathbb I_{in}\right)^2}{\sum_{i=k+1}^t  \mathbb I_{in}},
\end{equation}
Hence, for 0-1  matrices, the sketching procedure based on log-GLR statistic is given by 
\begin{equation}
\begin{split}
T_{\rm \{0,1\}}= &\\
 \inf\{t: &
\max_{t-w\leq k < t} \frac{1}{2}\sum_{n=1}^N \frac{\left(\sum_{i=k+1}^t [x_i]_n \mathbb I_{in}\right)^2}{\sum_{i=k+1}^t  \mathbb I_{in}} > b\},
\end{split}
\label{proc3}
\end{equation}
where $b>0$ is the prescribed threshold and $w$ is the window length. Note that the log-GLR statistic essentially computes the sum of each entry within the time window $[t-w,t)$, and then averages the squared-sum.

\section{Performance analysis}
\label{sec:analysis}

In this section, we present theoretical approximations to two performance metrics, the average-run-length (ARL), which captures the false-alarm-rate, and the expected detection delay (EDD), which captures the power of the detection statistic. 

\subsection{Performance metrics}

We first introduce some necessary notations. Under the null hypothesis in (\ref{hypothese}), the observations are zero mean. Denote the probability and expectation in this case by $\mathbb{P}^{\infty}$ and $\mathbb{E}^{\infty}$, respectively. Under the alternative hypothesis, there exists a change-point $\kappa$, $0\leq \kappa < \infty$ such that the observations have mean $\mu$ for all $t>\kappa$. Probability and expectation in this case are denoted by $\mathbb{P}^{\kappa}$ and $\mathbb{E}^{\kappa}$, respectively.

The choice of the threshold $b$ involves a trade-off between two standard performance metrics that are commonly used for analyzing change-point detection procedures \cite{XieSiegmund2012}: (i) the ARL, defined to be the expected value of the stopping time when there is
no change;
(ii) the EDD, defined to be
the expected stopping time in the extreme case where
a change occurs immediately at $\ka = 0$, which is denoted as $\mathbb{E}^0\{T\}$. 

\yc{The following argument from \cite{SiegmundVenkat1996} explains why we consider $\mathbb{E}^0\{T\}$. When there is a change at $\kappa$, we are interested in the expected delay until its detection, i.e., the conditional expectation $\mathbb{E}^{\kappa}\{T-\kappa|T > \kappa\}$, which is a function of $\kappa$. When the shift in the mean only occurs in the positive direction $[\mu]_i \geq 0$, it can be shown that $\sup_\kappa \mathbb{E}^{\kappa}\{T-\kappa|T > \kappa\} = \mathbb{E}^0\{T\}$. It is not obvious that this remains true when $[\mu]_i$ can be either positive or negative. However, since $\mathbb{E}^0\{T\}$ is certainly of interest and reasonably easy to analyze, it is common to consider $\mathbb{E}^0\{T\}$ in the literature and we also adopt this as a surrogate.}

\subsection{Fixed projection}

Define a special function (cf. \cite{Siegmund1985}, page 82) \[\nu(u) = 2u^{-2} \exp[-2\sum_{i=1}^\infty i^{-1} \Phi(-|u|i^{1/2}/2)],\] where $\Phi$ denotes the cumulative probability function (CDF) for the standard Gaussian with zero mean and unit variance. For numerical purposes, a simple and accurate approximation is given by (cf. \cite{SiegmundYakir2007})
\[
\nu(u) \approx \frac{2/u[\Phi(u/2) - 0.5]}{(u/2)\Phi(u/2) + \phi(u/2)},
\]
where $\phi$ denotes the probability distribution function (PDF) for standard Gaussian. We obtain an approximation to the ARL of the sketching procedure with a fixed projection as follows:
\begin{thm}[ARL, fixed projection]
\label{thmARL}
Assume that $1\leq M\leq N$, $b \rightarrow \infty$ with $M \rightarrow \infty$ and $b/M$ fixed.
Then for $w = o(b^r)$ for some
positive integer $r$, we have that the ARL of the sketching procedure defined in (\ref{proc}) is given by
\begin{equation}
\begin{split}
&\mathbb{E}^\infty\{T\} = \\
&\frac{2\sqrt{\pi}}{c(M, b, w) }  \frac{1}{1-\frac{M}{2b}} \frac{1}{\sqrt{M}}
\left(\frac{M}{2b}\right)^{\frac{M}{2}} e^{b-\frac{M}{2}}(1 + o(1)), \label{ET}
\end{split}
\end{equation}
where
\begin{equation}
c(M, b, w)  = \int_{\sqrt{\frac{2b}{w}}(1-\frac{M}{2b})}^{\sqrt{2b}(1-\frac{M}{2b})} u \nu^2(u) du.\label{c_def}
\end{equation}
\end{thm}
\yang{This theorem gives an explicit expression for ARL as a function of the threshold $b$, the dimension of the sketches $M$, and the window length $w$.  As we will show below, the approximation to ARL given by Theorem \ref{thmARL} is highly accurate. On a higher level, this theorem characterizes the mean of the stopping time, when the detection statistic is driven by noise. The requirement for $w= o(b^r)$ for some positive integer $r$ comes from \cite{XieSiegmund2012} that our results are based on; this ensures the correct scaling when we pass to the limit. This essentially requires that the window length be large enough when the threshold $b$ increases. In practice, $w$ has to be large enough so that it does not cause a miss detection: $w$ has to be longer than the anticipated expected detection delay as explained in \cite{XieSiegmund2012}.}

Moreover, we obtain an approximation to the EDD of the sketching procedure with a fixed projection as follows. Define
\begin{equation}
\Delta = \|V\transpose \mu\|.  \label{delta_matrix}
\end{equation}
Let $\tilde{S}_t\triangleq \sum_{i=1}^t\delta_i$ be a random walk where the increments $\delta_i$ are independent and identically distributed with mean $\Delta^2/2$ and variance $\Delta^2$. Since the random walk has a positive drift, its minimum value is lower bounded by zero and we can find the expected value of the minimum is given by \cite{SiegmundYakir2007}
 \[\mathbb{E}\{\min_{t\geq 0} \tilde{S}_t\} = -\sum_{i=1}^\infty i^{-1}\mathbb{E}\{\tilde{S}_i^-\},\] where $(x)^- = - \min\{x, 0\}$, and the infinite series  converges and can be evaluated easily numerically. 
 Define a stopping time \[\tau = \min\{t: \tilde{S}_t > 0\}.\] Let $\rho(\Delta) = \frac{1}{2}\mathbb E\{\tilde S_\tau^2\}/\mathbb E \{\tilde{S}_\tau\}$. It can be shown that \cite{XieSiegmund2012} 
\[\rho(\Delta) = \Delta^2/4 + 1 - \sum_{i=1}^\infty i^{-1}\mathbb{E}\{\tilde{S}_i^-\}.
\]
%
%So given $\Delta$,  $\rho (\Delta)$ can be easily evaluated numerically.  
\begin{thm}[EDD, fixed projection]
\label{thmEDD}
Suppose $b \rightarrow \infty$ with other parameters held fixed.
Then for a given matrix $A$ with the right singular vectors $V$, the EDD of the sketching procedure (\ref{proc}) when $\ka = 0$ is given by
\begin{equation}
\mathbb{E}^0\{T\} =
\frac{b+\rho(\Delta) - M/2 - \mathbb{E}\{\min_{t\geq 0} \tilde{S}_t\} + o(1)}{\Delta^2/2},
\label{DD}
\end{equation}
\end{thm}

\yang{The theorem finds an explicit expression for the EDD as a function of threshold $b$, the number of sketches $M$, and the signal strength captured through $\Delta$ which depends on the post mean vector $\mu$ and the projection subspace $V$. }

The proofs for the above two theorems utilize the equivalent form of $T$ in (\ref{proc2}) and draw a connection of the sketching procedure to the so-called mixture procedure (cf. $T_2$ in \cite{XieSiegmund2012}) when  $M$ sensors are affected by the change, and the post-change mean vector is given by $V\transpose \mu$.

\subsubsection{Accuracy of theoretical approximations}

Consider a $A$ generated as a Gaussian random matrix, with entries i.i.d. $\mathcal{N}(0, 1/N)$. Using the expression in Theorem \ref{thmARL}, we can find the threshold $b$ such that the corresponding ARL is equal to 5000. \yc{This can be done conveniently, since the ARL is an increasing function of the threshold $b$, we can use bisection to find such a threshold $b$.} Then we compare it with a threshold $b$ found from the simulation. 

As shown in Table \ref{table_comp}, the threshold found using Theorem \ref{thmARL} is very close to that obtained from simulation. \yc{Therefore, even if the theoretical ARL approximation is derived for $N$ tends to infinity, it is still applicable when $N$ is large but finite}. Theorem \ref{thmARL} is quite useful in determining a threshold for a targeted ARL, as simulations for large $N$ and $M$ can be quite time-consuming, especially for a large ARL (e.g., 5000 or 10,000).

Moreover, we simulate the EDD for detecting a signal such that the post-change mean vector $\mu$ has all entries equal to a constant $[\mu]_i = 0.5$.
As also shown in Table \ref{table_comp}, the approximation for EDD using Theorem \ref{thmEDD} is quite accurate.

We have also verified that the theoretical approximations are accurate for the expander graphs  and details omitted here since they are similar.
\begin{table}[h!]
\caption{$A$ being a fixed Gaussian random matrix. $N$ = 100, $w = 200$, ARL = 5000, for simulated EDD $[\mu]_i = 0.5$.
\yc{Numbers in the parentheses are the standard deviation of the simulated EDD.} }
\begin{center}
\begin{tabular}{c||c|c||c|c}
\hline
$M$ & $b$ (theo)  & $b$ (simu) & EDD (theo) & EDD (simu) \\\hline
100 & 84.65 & 84.44 & 3.4 & 4.3 (0.9) \\
70 & 64.85 & 64.52 & 4.0 & 5.1 (1.2)\\
50 & 51.04 & 50.75 & 4.8 & 5.9 (1.6) \\
30 & 36.36 & 36.43 & 7.7 & 7.6 (2.5)\\
10 & 19.59 & 19.63 & 19.8 & 17.4 (9.8)\\\hline
\end{tabular}
\label{table_comp}
\end{center}
\end{table}

\subsubsection{Consequence}
Theorem \ref{thmARL} and Theorem \ref{thmEDD} have the following consequence:
\begin{remark}\label{remark_b_vs_M}
For a fixed large ARL, when $M$ increases, \yang{the ratio $M/b$ is bounded between $0.5$ and $2$}. This is a property quite useful for establishing results in Section \ref{sec:A}. This is demonstrated numerically in Fig. \ref{fig:b_vs_M} when $N = 100$, $w = 200$, for a fixed ARL being 5000. The corresponding threshold $b$ is found using Theorem \ref{thmARL}, when $M$ increases from 10 to 100. %(In Section \ref{sec:numerical}, we verify that Theorem \ref{thmARL} is an accurate approximation.)  
More precisely, Theorem \ref{thmARL} leads to the following corollary:
\begin{corollary}
Assume a large constant $\gamma \in (e^5, e^{20})$. Let $w\geq 100$. For any large enough $M>24.85$, the threshold $b$ such that the corresponding ARL $\mathbb{E}^{\infty}\{T\} =\gamma$ satisfies $M/b \in (0.5,2)$. In other words, $\max\{M/b, b/M\} \leq 2$. 
\label{M_over_b_range}
\end{corollary}
Note that $e^{20}$ is on the order of $5\times10^{8}$, hence, it means that ARL can be very large, however, it is still bounded above (this means that the corollary holds for an non-asymptotic regime). % (rather than having it goes to infinity. 

\begin{figure}
\begin{center}
\includegraphics[width = 0.7\linewidth]{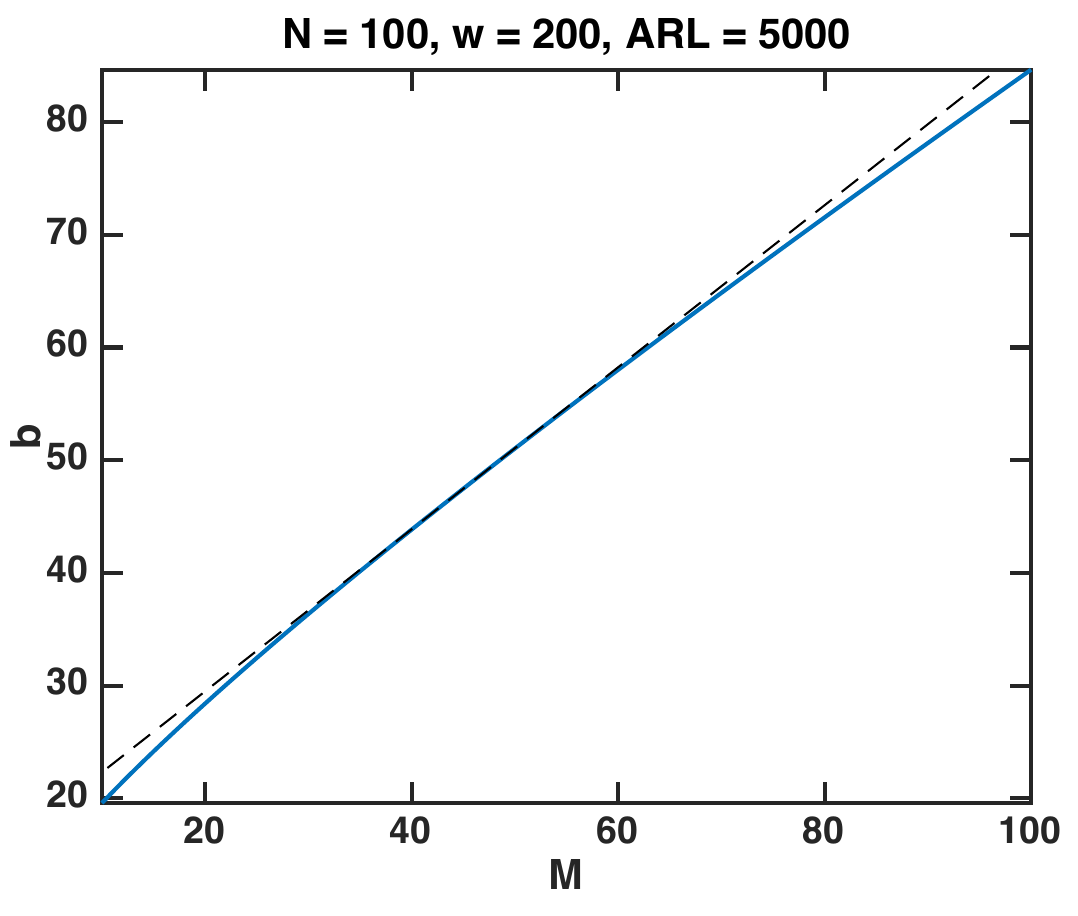}
\end{center}
\caption{For a fixed ARL being 5000, the threshold $b$ versus $M$ obtained using Theorem \ref{thmARL}, when $N = 100$ and $w = 200$. The dashed line corresponds to a tangent line of the curve at one point.}
\label{fig:b_vs_M}
\end{figure}

\begin{remark}
As $b\rightarrow\infty$, the first order approximation to the EDD  in Theorem \ref{thmEDD} is given by $b/(\Delta^2/2)$, i.e., the threshold $b$ divided by the Kullback-Leibler (K-L) divergence (see, e.g., \cite{XieSiegmund2012} shows that $\Delta^2/2$ is the K-L divergence between $\mathcal{N}(0, I_M)$ and $\mathcal{N}(V\transpose \mu, I_M)$). This is consistent with our intuition since the expected increment of the detection statistics is roughly the K-L divergence of the test. \yang{For finite $b$, especially when the signal strength is weak and when the number of sketches $M$ is not large enough, the other terms than $b/(\Delta^2/2)$ will play a significant role in determining the EDD.}
\end{remark}

\subsection{Time-varying 0-1 random projection matrices}

Below, we obtain approximations to ARL and EDD for $T_{\{0,1\}}$, i.e., when 0-1 sketching matrices are used. We assume a fixed dimension $M_t=M, \forall t >0$. We also assume that at each time $t$, we randomly select $M$ out of $N$ dimensions to observe. Hence, at each time, each signal dimension has a probability 
\[r=M/N \in (0, 1)\] to be observed. The sampling scheme is illustrated in Fig. \ref{fig:binning}, when $N=10$ and $M = 3$ (the number of the dots in each column is $3$) over 17 consecutive time periods from time $k=t-17$ to time $t$. 

For such a sampling scheme, we have the following result:
\begin{thm}[ARL, time-varying 0-1 random projection]\label{thm_time_varying}
Let $r = M/N$.  Let $b'=b/r$. When $b\rightarrow \infty$, for the procedure defined in (\ref{proc3}), we have that
\begin{equation}
\begin{split}
&\mathbb{E}^\infty\{T_{\{0, 1\}}\} \\
&= \frac{2\sqrt{\pi}}{c(N, b', w) } \frac{1}{\sqrt{N}} \frac{1}{1-\frac{N}{2b'}}
\left(\frac{N}{2}\right)^{\frac{N}{2}} b'^{-\frac{N}{2}} e^{b'-\frac{N}{2}}+o(1),
\end{split} \label{ET2}
\end{equation}
where $c(N, b', w)$ is defined by replacing $b$ with $b'$ in (\ref{c_def}).
\end{thm}
%The approximation is obtained by an asymptotic analysis and replacing the detection statistic by its mean over random sampling conditioning on the data. 
\begin{figure}
\begin{center}
\includegraphics[width = .7\linewidth]{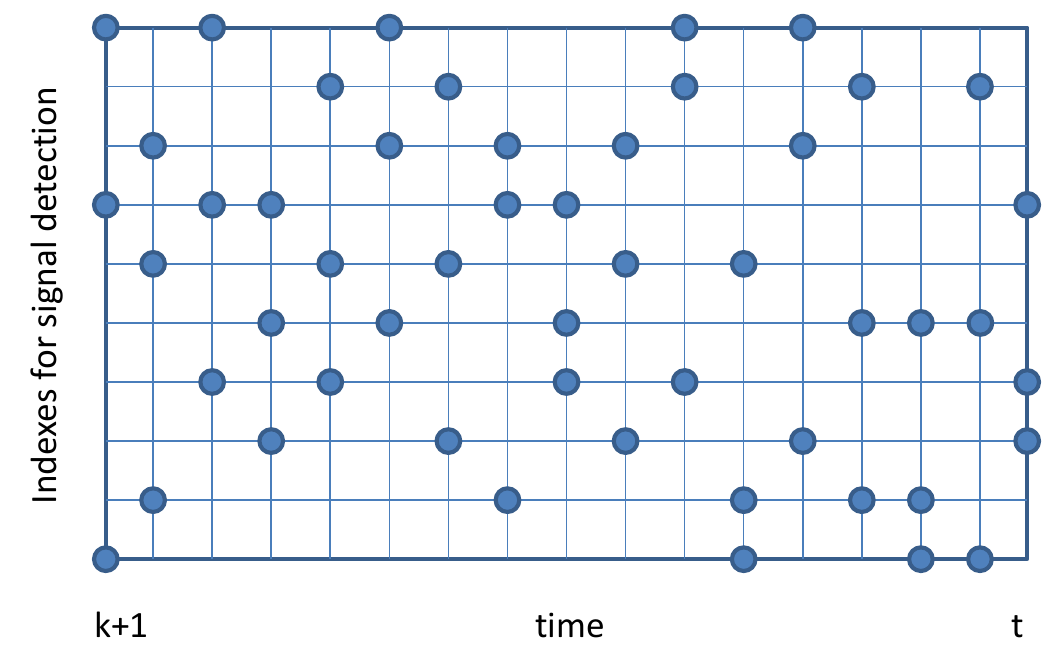}
\end{center}
\caption{A sampling pattern when $A_t$ is a $0$-$1$ matrix, $M = 3$, and $N = 10$. Dot represent entries being observed.}
\label{fig:binning}
\end{figure}

Moreover, we can obtain an approximation to EDD of $T_{\{0,1\}}$, as justified by the following arguments. First,  relax the deterministic constraint that at each time we observe exactly $M$ out of $N$ entries. Instead, assume a random sampling scheme that at each time we observe one entry of $x_i$ with probability $r$, $1\leq n\leq N$. 
Consider i.i.d. Bernoulli random variables $\xi_{ni}$ with parameter $r$ for $1\leq n\leq N$ and $i\geq 1$. Define
\[
Z_{n,k,t} \triangleq \frac{\sum_{i=k+1}^t [x_i]_n\xi_{ni}}{\sqrt{(t-k)r}}.
\]
Based on this, we define a procedure whose behavior is arguably similar to $T_{\{0,1\}}$:
\[
T'_{\{0,1\}} = \inf \{t\geq 1: \max_{t-w\leq k < t} \frac{1}{2} \sum_{n=1}^N Z_{n,k,t}^2 > b \},
\]
where $b>0$ is the prescribed threshold. \yang{Then, using the arguments in Appendix \ref{justification}, we can show that the approximation to EDD of this procedure is given by}
\begin{equation}
\mathbb{E}^0\{ T_{\{0,1\}}' \} = \left(\frac{2b-N}{\sum_{n=1}^N \mu_n^2} + o(1)\right)\cdot \frac{N}{M},
\label{DD2}
\end{equation}
and we use this to approximate the EDD of $T_{\{0,1\}}$.
%\subsubsection{Accuracy of theoretical approximations}

Table \ref{table_comp2} shows the accuracy of the approximations for ARL in (\ref{ET2}) and for EDD in (\ref{DD2}) with various $M'$s when $N=100$, $w=200$, and all entries of $[\mu]_i =  0.5$. The results show that the thresholds $b$ obtained using the theoretical approximations, and that the EDD approximations are both  very accurate. 
%Note that the approximate is more accurate if $M/N$ is close to $1$. One intuitive explanation is that,  when $M/N$ is close to $0$, replacing the number of observations for each entries by its mean $(t-k)M/N$ is less accurate, and hence, the approximation by using $T_2'$ for $T_2$ is less accurate.

%\yc{To understand whether sparsity of the post-change mean vector affects the accuracy of the ARL approximation in (\ref{DD2}), we perform another experiment letting $N=100$, $w=200$ and  $25\%$ of $[\mu]_i = 1$. Note that the total signal energy $\sum_{n=1}^N \mu_n^2$  is the same as above, so the theoretical EDD computed from (\ref{DD2}) should be same. The simulated EDDs are shown in Table \ref{table_comp3}, which shows that the approximation in (\ref{DD2}) is still accurate.}

\begin{table}[h!]
\caption{$A_t'$s being time-varying. $N$ = 100, $w = 200$, ARL = 5000, for simulated EDDs all $[\mu]_i = 0.5$.
\yc{Numbers in the parentheses are standard deviation of the simulated results.}}
\begin{center}
\begin{tabular}{c||c|c||c|c}
\hline
$M$ & $b$ (theo)  & $b$ (simu) & EDD (theo) & EDD (simu) \\\hline
100 & 84.65 & 84.44 & 2.8 & 3.3 (0.8) \\
70 & 83.72 & 83.41 & 3.8 & 4.5 (1.2)\\
50 & 82.84 & 83.02 & 5.3 & 6.1 (1.5) \\
30 & 81.46 & 82.48 & 8.7 & 9.8 (2.4)\\
10 & 78.32 & 79.27 & 23.4 & 26.6 (6.4) \\\hline
\end{tabular}
\label{table_comp2}
\end{center}
\end{table}

%\begin{table}[h!]
%\caption{$A_t'$s being time-varying. $N$ = 100, $w = 200$, ARL = 5000, for simulated EDDs $25\%$ entries $[\mu]_i = 1$.
%\yc{Number in the parentheses are standard deviation of the simulated results.}}
%\begin{center}
%\begin{tabular}{c||c|c||c|c}
%\hline
%$M$ & $b$ (theo)  & $b$ (simu) & EDD (theo) & EDD (simu) \\\hline
%100 & 84.65 & 84.44 & 2.8 & 3.3 (0.9) \\
%70 & 83.72 & 83.41 & 3.8 & 4.5 (1.2)\\
%50 & 82.84 & 83.02 & 5.3 & 6.1 (1.6) \\
%30 & 81.46 & 82.48 & 8.7 & 9.8 (2.5)\\
%10 & 78.32 & 79.27 & 23.4 & 26.7 (6.8) \\\hline
%\end{tabular}
%\label{table_comp3}
%\end{center}
%\end{table}

%When $b\rightarrow \infty$, in (\ref{ET}), $c(M, b, w)$ and $1/[1-M/(2b)]$ tend to constants, and the remaining terms can be written as
%\begin{equation}
%\exp\{ b[1 - (1-\log\frac{M}{2b} ) \frac{M}{2b} -\frac{\log M}{2b}] \}.
%\label{explanation}
%\end{equation}
%A common property of sequential change-point detection procedures is that the ARL grows exponentially with the threshold $b$ \cite{changepoint_new_book2014}, since we can achieve a large ARL for a reasonably valued threshold $b$. To achieve this, due to (\ref{explanation}), we need $b$ to grow at least linearly with $M$. }
%and the term that multiplies $b$ in the exponent tends to a positive constant. This property is desirable since we can obtain a large ARL for a reasonably valued threshold $b$.
%Hence, the remaining terms also tend to a constant under such a condition. This leads to a reasonable ARL expression. (\textcolor{red}{also needs further explanation ; the root has reasonable })

\end{remark}

\section{Bounding relative performance}\label{sec:A}

%There is a trade-off in the performance and complexity of sketching procedure: when using more sketches (in the extreme case, $M = N$), the performance of the sketching procedure becomes better, however, at the cost of higher complexity (since there is less dimensionality reduction). 
In this section, we characterize the relative performance of the sketching procedure compared to that without sketching (i.e., using the original log-GLR statistic). We show that the performance loss due sketching can be small, \yang{when the signal-to-noise ratio and $M$ are both sufficiently large.} In the following, we focus on fixed projection to illustrate this point. 

\subsection{Relative performance metric}\label{rel}

We consider a relative performance measure, which is the ratio of EDD using the original data (denoted as EDD($N$), which corresponds to $A = I$), versus the EDD using the sketches (denoted as EDD($M$) \[\frac{\mbox{EDD}(N)}{\mbox{EDD}(M)} \in (0, 1).\] We will show that this ratio depends critically on the following quantity 
\begin{equation}\label{gamma_def}
\Gamma\triangleq \frac{\|V\transpose \mu\|^2}{\|\mu\|^2},
\end{equation}
which is the ratio of the KL divergence after and before the sketching. 

We start by deriving the relative performance measure using  theoretical approximations we obtained in the last section. Recall the expression for ARL approximation in (\ref{DD}). \yang{Define  
\begin{equation}
h(\Delta, M) = \rho(\Delta) - M/2 - \mathbb{E}\{\tilde{S}_i^-\}. \label{h_def}
\end{equation}} From Theorem \ref{thmEDD}, we obtain that the EDD of the sketching procedure  is proportional to \yang{
\[
\frac{2b}{\|V\transpose \mu\|^2}\cdot \left(1+\frac{h(\|V\transpose \mu\|, M)}{2b}\right) \cdot (1 + o(1)).
\]}
\yang{Let $b_N$ and $b_M$ be the thresholds such that the corresponding ARLs are $5000$, for the  procedure without sketching and with $M$ sketches, respectively.} Define $Q_M = M/b_M$, $Q_N = N/b_N$ and \yang{
\begin{equation}
P = \frac{1+h(\|\mu\|, N)/b_N}{1+ h(\|V\transpose \mu\|, M)/b_M }.
\label{P}
\end{equation}}
Using the definitions above, we have
\begin{equation}
\begin{split}
\frac{\mbox{EDD}(N)}{\mbox{EDD}(M)} &= \yang{P \cdot}\yang{\frac{b_N}{b_M}} \cdot
\frac{\|V\transpose\mu\|^2}{\|\mu\|^2} (1+o(1)) \\
&=  \yang{P \cdot} \frac{N}{M} \cdot \frac{Q_M}{Q_N} \cdot \Gamma (1+o(1)).
\end{split}
\label{ratio1}
\end{equation}

\subsection{Discussion of factors in (\ref{ratio1})}

%Now we can bound the factors that the relative performance measure (\ref{ratio1}) depends on. 

\yang{We can show that $P\geq 1$ for sufficiently large $M$ and large signal strength. This can be verified numerically. Since all quantities that $P$ depends on can be computed explicitly: the thresholds $b_N$ and $b_M$ can be found from Theorem \ref{thmARL} once we set a target ARL, the  $h$ function can be evaluated using (\ref{h_def}) which  depends explicitly on $\Delta$ and $M$. 
Fig. \ref{fig:P} shows the value of $P$ when $N=100$ and all the entries of the post-change mean vector $\mu_i$ are equal to a constant value that varies across the $x$-axis. Note that $P$ is less than one only when the signal strength $\mu_i$ are small and $M$ is small. 
Thus, we have, \[\frac{\mbox{EDD}(N)}{\mbox{EDD}(M)} \geq \frac N M \cdot \frac{Q_M}{Q_N} \cdot \Gamma (1+o(1)),\] for sufficiently large $M$ and signal strength $\Delta$.}
\begin{figure}
\begin{center}
\includegraphics[width = .9\linewidth]{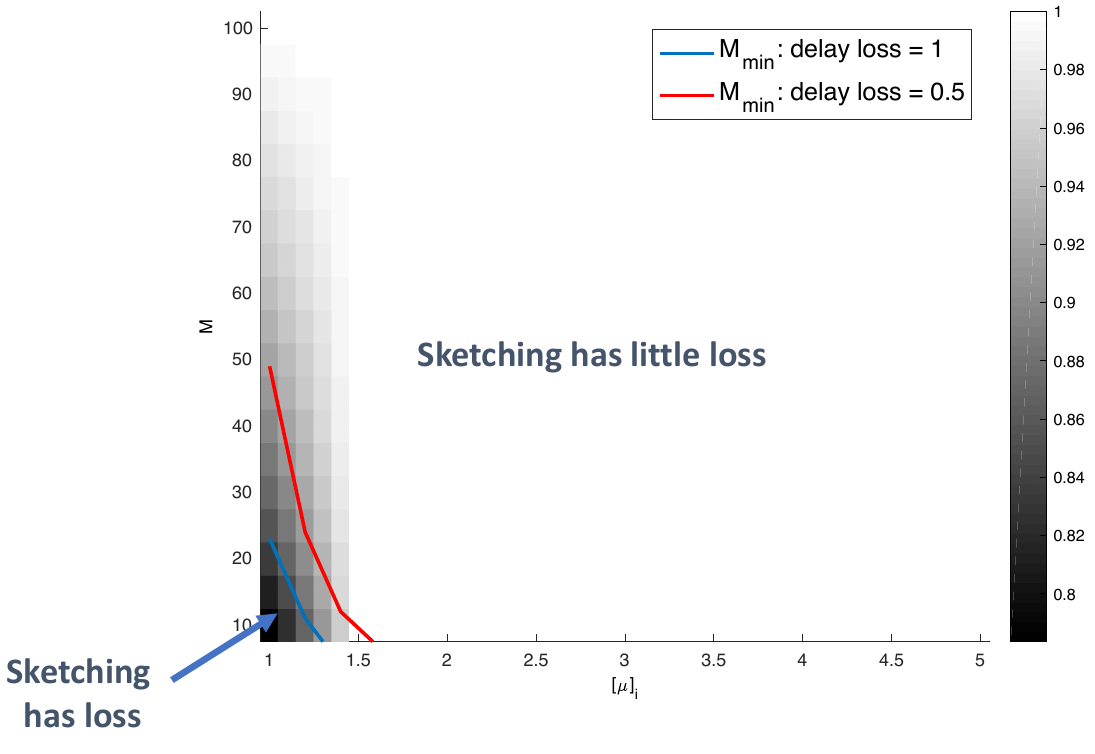}
\end{center}
\caption{\yang{The $P$ factor defined in (\ref{P}) for different $M$ and $[\mu]_i$, when the post-change mean vector has entries all equal to $[\mu]_i$. Assume $N=100$. The white regions correspond to $P \geq 1$, and dark regions correspond to $P<1$ and the darker, the smaller $P$ is (note that the smallest $P$ in this graph is above 0.75). We also plot the $M_{\min}$ (defined later in (\ref{Mmin})) required in these cases such that the EDD of the sketching procedure is no more than $\delta$ larger than the corresponding procedure without sketching (fixing ARL = 5000), for $\delta = 1$ and $\delta = 0.5$. The $M_{\min}$ are obtained by Monte Carlo simulation. The $M_{\min}$ versus $[\mu]_i$ correspond to the blue and the red curves, respectively. Above these two curves, the EDD with sketching is {\it almost the same} as before (without sketching), i.e., the regime where {\it sketching has little loss}. The left-bottom corner corresponds to the region where sketching has more loss. 
This also shows that indeed $P<1$ is an indicator of significant performance loss using sketching. }}
\label{fig:P}
\end{figure}

Using Corollary \ref{M_over_b_range}, we have that $Q_M \in (0.5,2)$ and $Q_N \in (0.5,2)$, and hence, \yang{an lower bound of} the ratio $\mbox{EDD}(N)/\mbox{EDD}(M)$ is 
between $(1/4)  (N/M)  \Gamma$ and $4  (N/M)  \Gamma$, \yang{for large $M$ or large signal strength}. 

Next, we will bound $\Gamma$ when $A$ is a Gaussian matrix and an expander graph, respectively.

\subsection{Bounding $\Gamma$}
\label{sec:choice}

\subsubsection{Gaussian matrix} \label{sec:Gaussian}

Consider $A\in \mathbb{R}^{M\times N}$ whose entries are i.i.d. Gaussian
with  zero mean and variance equal to $1/M$. First, we have the following lemma
\begin{lemma}[\cite{ruben79}]\label{quotient_dist}
Let $A\in\mathbb{R}^{M\times N}$ have i.i.d. $\mathcal{N}(0, 1)$ entries. Then for any fixed vector $\mu$, we have that
\begin{equation}\label{quotient_result}
\Gamma \sim\mbox{Beta}\left(\frac{M}{2},\frac{N-M}{2}\right).
\end{equation}
\end{lemma}
\yc{More related results can be found in \cite{betaDistribution90}.}
Since the $\mbox{Beta}(\alpha, \beta)$ distribution has a mean $\alpha/(\alpha+\beta)$, we have that
\[
\mathbb{E}\left\{\Gamma\right\} = \frac{M/2}{M/2 + (N-M)/2} = \frac{M}{N}.
\]
We may also show that, provided $M$ and $N$ grow proportionally, $\Gamma$ converges to its mean value at a rate exponential in $N$. Define $\delta\in(0,1)$ to be
\begin{equation}\label{delta_def}
\delta\triangleq\displaystyle\lim_{N\rightarrow\infty}\frac{M}{N}.
\end{equation}
We have the following result. 
\begin{thm}[Gaussian $A$]\label{mean_concentration}
Let $A\in\mathbb{R}^{M\times N}$ have entries i.i.d. $\mathcal{N}(0, 1)$.  
Let $N\rightarrow\infty$ such that (\ref{delta_def}) holds. Then  for $0<\epsilon<\min(\delta,1-\delta)$, we have that
\begin{equation}\label{concentration_result}
\mathbb{P}\left\{\delta-\epsilon<\Gamma<\delta+\epsilon\right\}\rightarrow 1,
\end{equation}
at a rate exponential in $N$. \yang{Hence, for  Gaussian $A$, $\Gamma$ is approximately $M/N$ with probability $1$.}
\end{thm}

\begin{figure}
\begin{center}
\includegraphics[width = .5\linewidth]{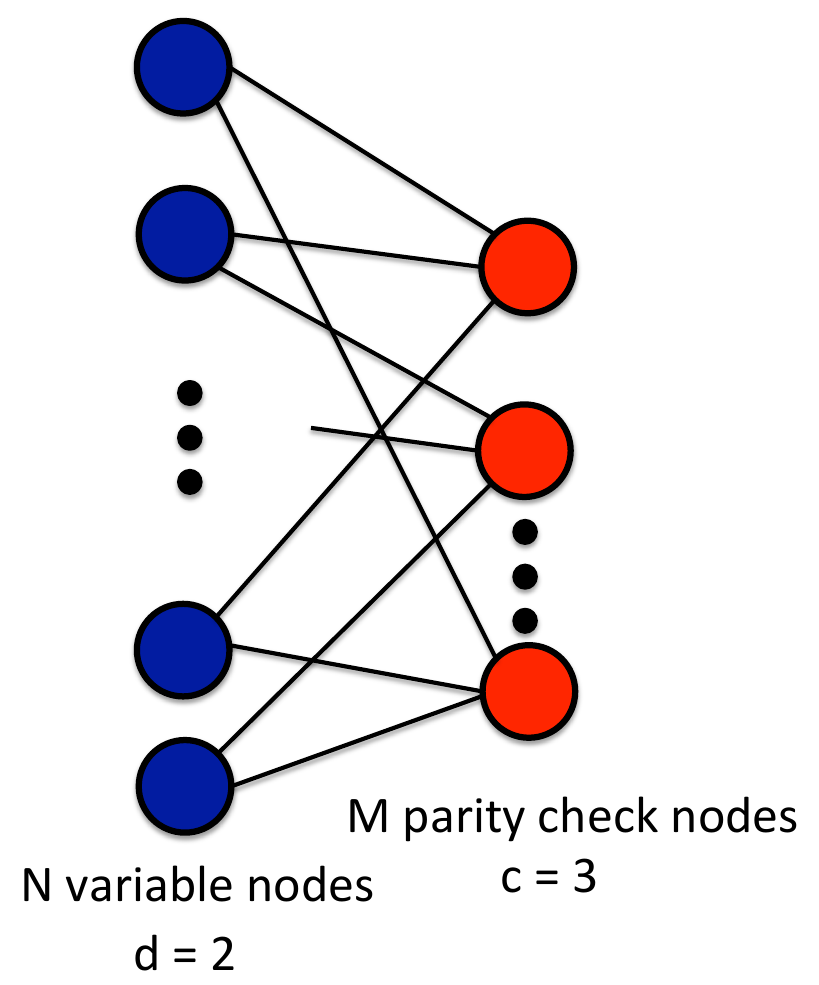}
\end{center}
\caption{Illustration of a bipartite graph with $d=2$ and $c = 3$. Following coding theory terminology, we call the left variables nodes (there are $N$ such variables), which correspond to the entries of $x_t$, and the right variables parity check nodes (there are $M$ such nodes), which correspond to entries of $y_t$. In a bipartite graph, connections between the variable nodes are not allowed. The adjacency matrix of the bipartite graph corresponds to our $A$ or $A_t$. }
\label{fig:EG}
\end{figure}

\subsubsection{Sparse sketching matrix $A$}\label{sec:expander}

We can show that for certain sparse 0-1 matrices $A$ (in particular, the expander graphs), $\Gamma$ is also bounded. This holds for the ``one-sided'' changes, i.e., the post-change mean vector is element-wise positive. Such a scenario is encountered in environmental monitoring (see, e.g., \cite{XieSiegmund2012,S3T2017}). 
These sparse sketching matrices $A$ enable efficient sketching schemes, as each entry in the sketching vector only requires collecting information from few dimensions of the original data vector.

Assume $[\mu]_i\geq 0$ for all $i$. Let $A \in \mathbb{R}^{M\times N}$ be consisting of binary entries, which corresponds to a bipartite graph, illustrated in Fig. \ref{fig:EG}. 
We further consider a bipartite graph with regular left degree $c$ (i.e., the number of edges from each variable node is $c$), and regular right degree $d$ (i.e., the number of edges from each parity check node is $d$), as illustrated in Fig. \ref{fig:EG}. Hence, this requires $Nc = Md$. 

Expander graphs satisfy the above requirements, and they have been used in compressed sensing to sense a sparse vector (e.g., \cite{XuHassibi2007}). In particular, a matrix $A$ corresponds to a $(s,\epsilon)$-expander graph with regular right degree $d$ if and only if each column  of $A$ has exactly $d$ ``1''s, and for any set $S$ of right nodes with $|S|\leq s$, the set of neighbors $\mathcal{N}(S)$ of the left nodes has size $\mathcal{N}(S)\geq (1-\epsilon) d |S|$. If it further holds that each row of $A$ has $c$ ``1''s,  we say $A$ corresponds to a  $(s,\epsilon)$-expander with regular right degree $d$ and regular left degree $c$.
The existence of such expander graphs is established in \cite{BM01}:
\begin{lemma}[\cite{BM01}]\label{lem:goodexpander}
For any fixed $\epsilon >0$ and $\rho \triangleq M/N <1$, when $N$ is sufficiently large, there always exists an $(\alpha N, \epsilon)$ expander with a regular right degree $d$ and a regular left degree $c$ for some constants $\alpha\in(0,1)$, $d$ and $c$.
\end{lemma}

\begin{thm}[Expander $A$]\label{thm:expander}
If $A$ corresponds to a $(s, \epsilon)$-expander with regular degree $d$ and regular left degree $c$, for any nonnegative vector $[\mu]_i\geq 0$, we have that
\[
\Gamma \geq \frac{M(1-\epsilon)}{dN}.
\]
\end{thm}
\yang{Hence, for expander graphs, $\Gamma$ is approximately greater than $M/N\cdot (1/d)$, where $d$ is a small number.}
\subsection{Consequence}

Combine the results above, we  shown that for the regime \yang{where $M$ and the signal strength are sufficiently large}, the performance loss can be small (as indeed observed from our numerical examples). In this regime, when $A$ is a Gaussian random matrix, the relative performance measure $\mbox{EDD}(N)/\mbox{EDD}(M)$ is a constant, under the conditions in Corollary \ref{M_over_b_range}.  
Moreover, when $A$ is a sparse 0-1 matrix with $d$ non-zero entries on each row (in particular, an expander graph), the ratio (\ref{ratio1}) $\mbox{EDD}(N)/\mbox{EDD}(M)$ is lower bounded by  $(1/4) \cdot d/(1-\epsilon)$ 
for some small number $\epsilon > 0$, when Corollary \ref{M_over_b_range} holds.

There is one intuitive explanation. Unlike in compressed sensing, where the goal is to recover a sparse signal and one needs the projection to preserve norm up to a factor through the restricted isometry property (RIP) \cite{RIPCandes2008}, our goal is to detect a non-zero vector in Gaussian noise, which is a much simpler task than compressed sensing. Hence, even though the projection reduces the norm of the vector, as long as the projection does not diminish the signal normal below the noise floor. % \textcolor{blue}{Note that the loss in EDD is compensated to a certain extent by a lower threshold for the same ARL.}

\yang{On the other hand, when the signal is weak, and $M$ is not large enough, there can be significant performance loss (as indeed observed in our numerical examples) and we cannot lower bound the relative performance measure. Fortunately, in this regime, we can use our theoretical results in Theorem \ref{thmARL} and Theorem \ref{thmEDD} to design the number of sketches $M$ for an anticipated worst-case signal strength $\Delta$, or determine the infeasibility of the problem, i.e., it is better not to use sketching since the signal is too weak.}

\section{Numerical examples}
\label{sec:numerical}

In this section, we present numerical examples to demonstrate the performance of the sketching procedure. We focus on comparing the sketching procedure with the GLR procedure without sketching (by letting $A = I$ in the sketching procedure). %This is similar to what we did earlier in Section \ref{rel}.  
We also compare the sketching procedures with a standard multivariate CUSUM using sketches. 

In the subsequent examples, we select ARL to be 5000 to represent a low false detection rate (similar choice has been made in other sequential change-point detection work such as \cite{XieSiegmund2012}). In practice, however, the target ARL value depends on how frequent we can tolerate false detection (e.g., once a month or once a year). 
Below, EDD$_{\rm o}$ denotes the EDD when $A = I$ (i.e., no sketching is used). All simulated results are obtained from $10^4$ repetitions. \yang{We also consider the minimum number of sketches  
\begin{equation}
M_{\min}:  \quad\mbox{EDD}(M_{\min}) \leq \delta + \mbox{EDD}_{\rm o},
\label{Mmin}
\end{equation}
 such that the  corresponding sketching procedure is only $\delta$ sample slower than the full procedure. Below, we focus on the delay loss $\delta = 1$.}

\subsection{Fixed projection, Gaussian random matrix}
\label{sec:numerical_gaussian}

First consider Gaussian $A$ with $N=500$ and different number of sketches $M < N$. 

\subsubsection{EDD versus signal magnitude}
Assume the post-change mean vector has entries with equal magnitude: $[\mu]_i = \mu_0$, to simplify our discussion. Fig. \ref{fig:EDD_comp}(a) shows EDD versus an increasing signal magnitude $\mu_0$. 
Note that when $\mu_0$ and \yang{$M$ are} sufficiently large, the sketching procedure can approach the performance of the procedure using the full data as predicted by our theory. \yang{When signal is weak, we have to use a much larger $M$ to prevent a significant performance loss (and when signal is too weak we cannot use sketching).} Table \ref{tab:minimumM_1} shows $M_{\min}$ \yang{for each signal strength}; we find that when $\mu_0$ is sufficiently large, we may even use $M_{\min}$ less than $30$ for an $N = 500$ to have little performance loss.  Note that here we do not require signals to be sparse. 

\begin{figure}[h]
\begin{center}
\begin{tabular}{c}
\includegraphics[width=0.7\linewidth]{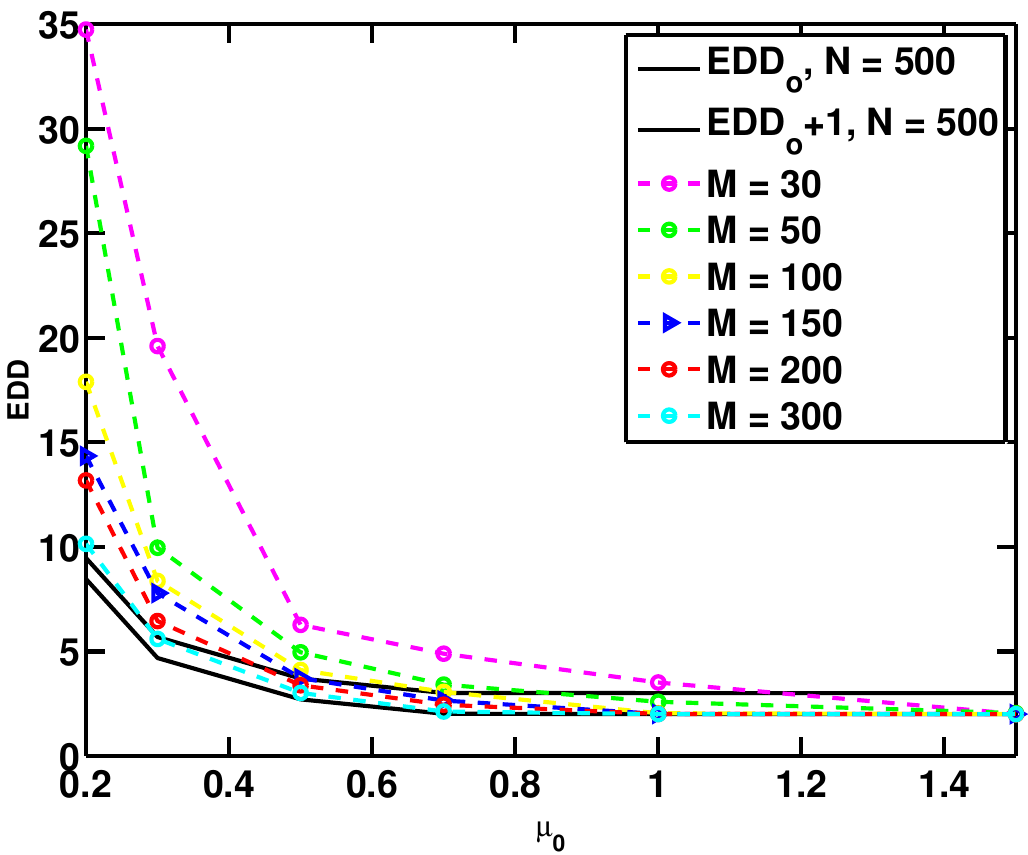}\\
(a) \\
\includegraphics[width=0.7\linewidth]{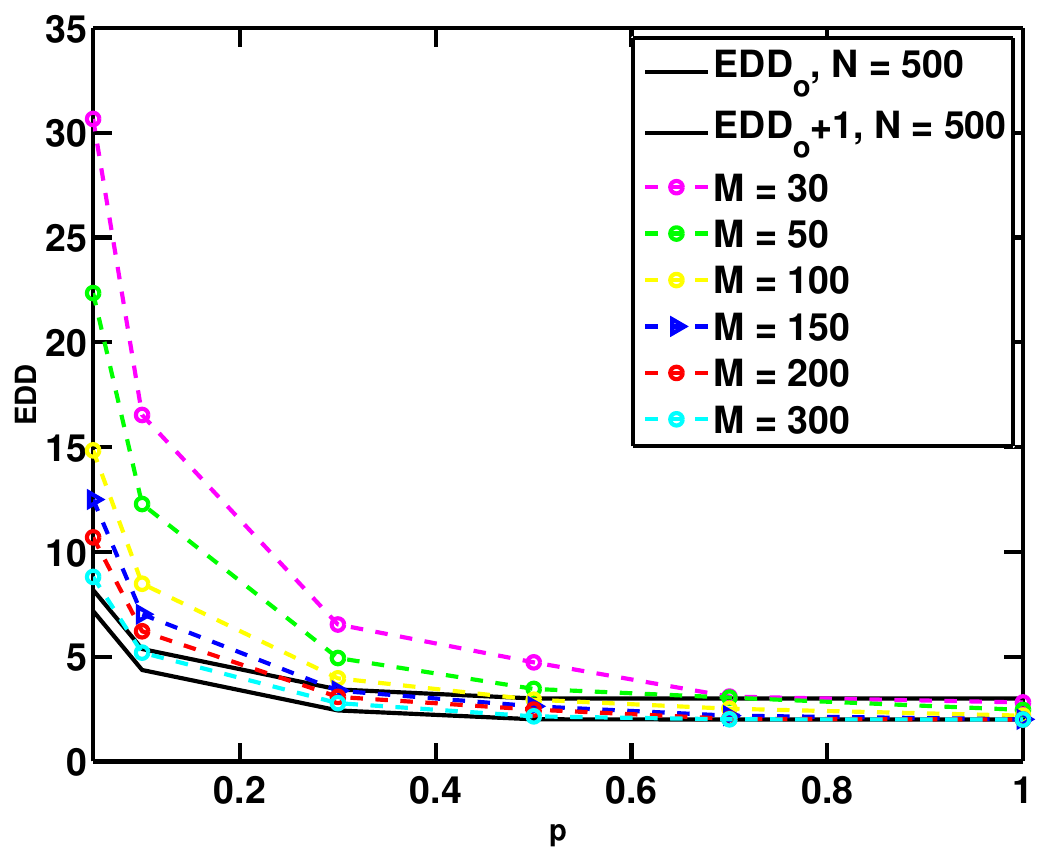}\\
(b)  %\\
%(a) \\
%\includegraphics[width=0.7\linewidth]{detectability}\\
%(b)
\end{tabular}
\end{center}
\caption{$A$ being a fixed Gaussian random matrix: \yc{The standard deviation of each point is less than half of its value.} (a) EDD versus $\mu_0$, when all $[\mu]_i = \mu_0$; (b) EDD versus $p$ when we randomly select $100p\%$ entries $[\mu]_i$ to be one and set the other entries to be zero; the smallest value of $p$ is 0.05.
}
\label{fig:EDD_comp}
\end{figure}

\subsubsection{EDD versus signal sparsity}

Now assume that the post-change mean vector is sparse:  only $100p\%$ entries $\mu_i$ being 1 and the other entries being 0. Fig. \ref{fig:EDD_comp}(b) shows EDD versus an increasing $p$. Note that as $p$ increases, the signal strength also increases, thus the sketching procedure will approach the performance using the full data. Similarly, the $M_{\min}$ required is listed in Table \ref{tab:minimumM_2}. For example, when $p=0.5$, we find that one can use $M_{\min}=100$ for an $N=500$ with little performance loss.

\begin{table}[h!]
\centering
\caption{Assume $A$ being a fixed Gaussian random matrix. The table shows $M_{\min}$ required for various mean shifts $\mu_0$ as shown in Fig. \ref{fig:EDD_comp}(a). Here $N=500$, $w=200$ and all $[\mu]_i = \mu_0$. 
\yc{Numbers in the parentheses are standard deviation of the simulated results.}}
\label{tab:minimumM_1}
\begin{tabular}{c|c|c|c|c|c}
\hline
               $\mu_0$      & $0.3$ & $0.5$ & $0.7$ & $1$ & $1.2$ \\ \hline
$M_{\min}$ & 300        & 150        & 100        & 50       & 30 \\ \hline
%$\mbox{EDD}_{\rm o}$ & {\footnotesize 8.5 (2.0)} & {\footnotesize 4.7 (0.9)} & {\footnotesize 2.7 (0.5)} & {\footnotesize 2.0 (1.1)} & {\footnotesize 2 (0.01)} \\ \hline
\end{tabular}
\end{table}

\begin{table}[h!]
\centering
\caption{$A$ being a fixed Gaussian random matrix. Minimum $M_{\min}$ required for various sparsity setting with parameter $p$ as shown in Fig. \ref{fig:EDD_comp}({b}). $N=500$, $w=200$ and $100p\%$ of entries $[\mu]_i = 1$.
\yc{Number in parentheses are standard deviation of the simulated results.}}
\label{tab:minimumM_2}
\begin{tabular}{c|c|c|c|c|c}
\hline
                  $p$   & $0.1$ & $0.2$ & $0.3$ & $0.5$ & $0.7$ \\ \hline
$M_{\min}$ & 300        & 200        & 150        & 100       & 50 \\ \hline
%$\mbox{EDD}_{\rm o}$ & {\footnotesize 4.4 (0.8)} & {\footnotesize 2.4 (0.5)} & {\footnotesize 2.0 (0.1)} & {\footnotesize 2.0 ($0.02$)} & {\footnotesize 2.0 ($0.01$)} \\ \hline
\end{tabular}
\end{table}

\subsection{Fixed projection, expander graph}

Now assume $A$ being an expander graph with $N=500$ and different number of sketches $M < N$. We run the simulations with the same settings as those in Section \ref{sec:numerical_gaussian}.

\subsubsection{EDD versus signal magnitude}

Assume the post-change mean vector $[\mu]_i = \mu_0$. Fig. \ref{fig:EDD_expander}(a) shows EDD with an increasing $\mu_0$. Note that the simulated EDDs are smaller than those for the Gaussian random projections in Fig. \ref{fig:EDD_comp}. A possible reason is that the expander graph is better at aggregating the signals when $[\mu]_i$ are all positive. However, when $[\mu]_i$ are can be either positive or negative, the two choices of $A$ have similar performance, as shown in Fig. \ref{fig:EDD_Twotype_M_posi_nega}, where $[\mu]_i$ are drawn i.i.d. uniformly from  $[-3, 3]$.

\subsubsection{EDD versus signal sparsity} 

Assume that the post-change mean vector has only $100p\%$ entries $\mu_i$ being one and the other entries being zero. Fig. \ref{fig:EDD_expander}(b) shows the simulated EDD versus an increasing $p$. As $p$ tends to $1$, the sketching procedure approaches the performance using the full data. 
%We obtain  $M_{\min}$ shown in Table \ref{tab:minimumM_2}. For example, when $p$ is around $0.5$, we may use $M_{\min}=50$ for an $N=500$ with little performance loss. When $p$ is larger than $0.7$, one may use $M_{\min}$ less than $30$ for an $N=500$. 

\begin{figure}[h]
%\vspace{-0.2in}
\begin{center}
\begin{tabular}{c}
\includegraphics[width=0.7\linewidth]{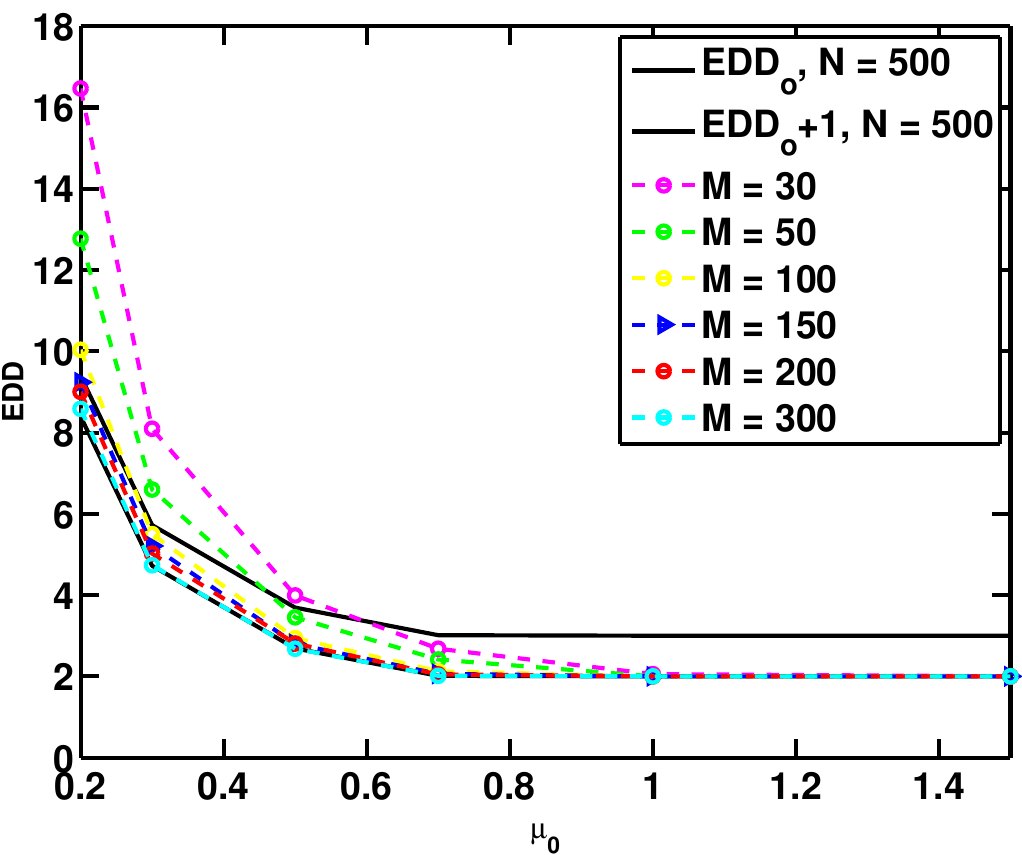}\\
(a)\\
\includegraphics[width=0.7\linewidth]{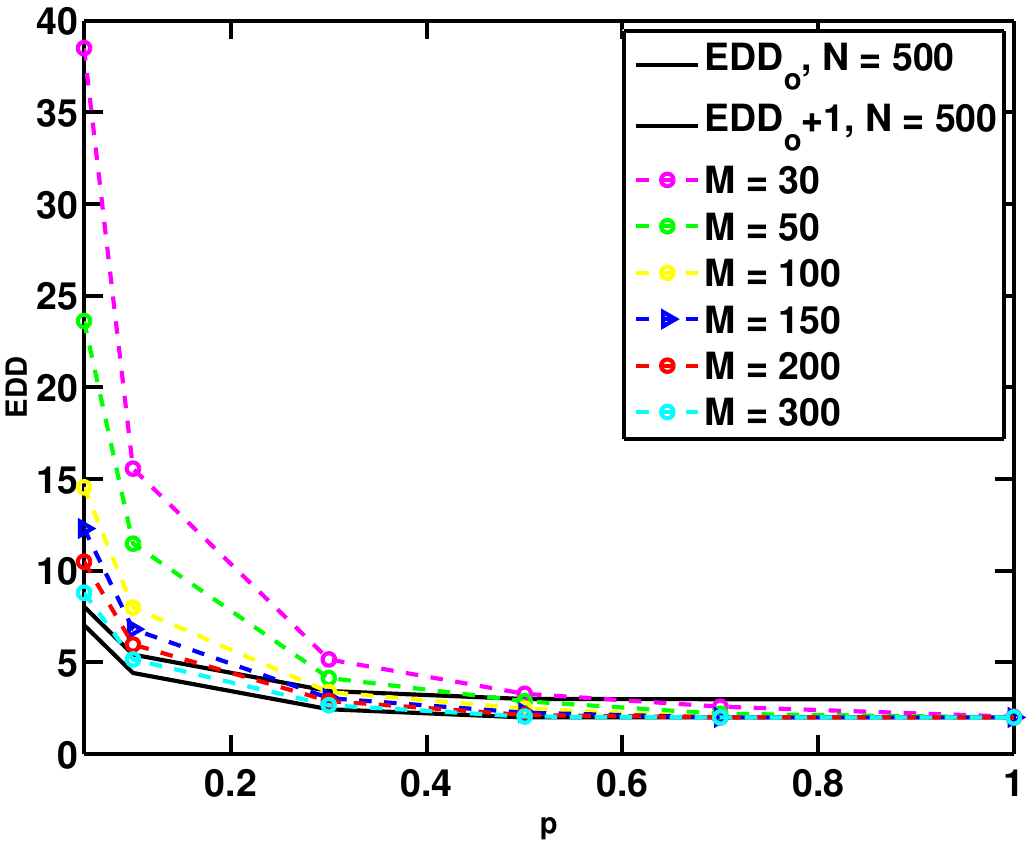}\\
(b)
\end{tabular}
\end{center}
\caption{$A$ being a fixed expander graph.  \yc{The standard deviation of each point is less than half of its value.} (a) EDD versus $\mu_0$,  when all $[\mu]_i = \mu_0$; (b) EDD versus $p$ when we randomly select $100p\%$ entries $[\mu]_i$ to be one and set the other entries to be zero; the smallest value of $p$ is 0.05.}
\label{fig:EDD_expander}
\end{figure}

\begin{figure}[h]
%\vspace{-0.2in}
\begin{center}
\begin{tabular}{c}
\includegraphics[width=0.7\linewidth]{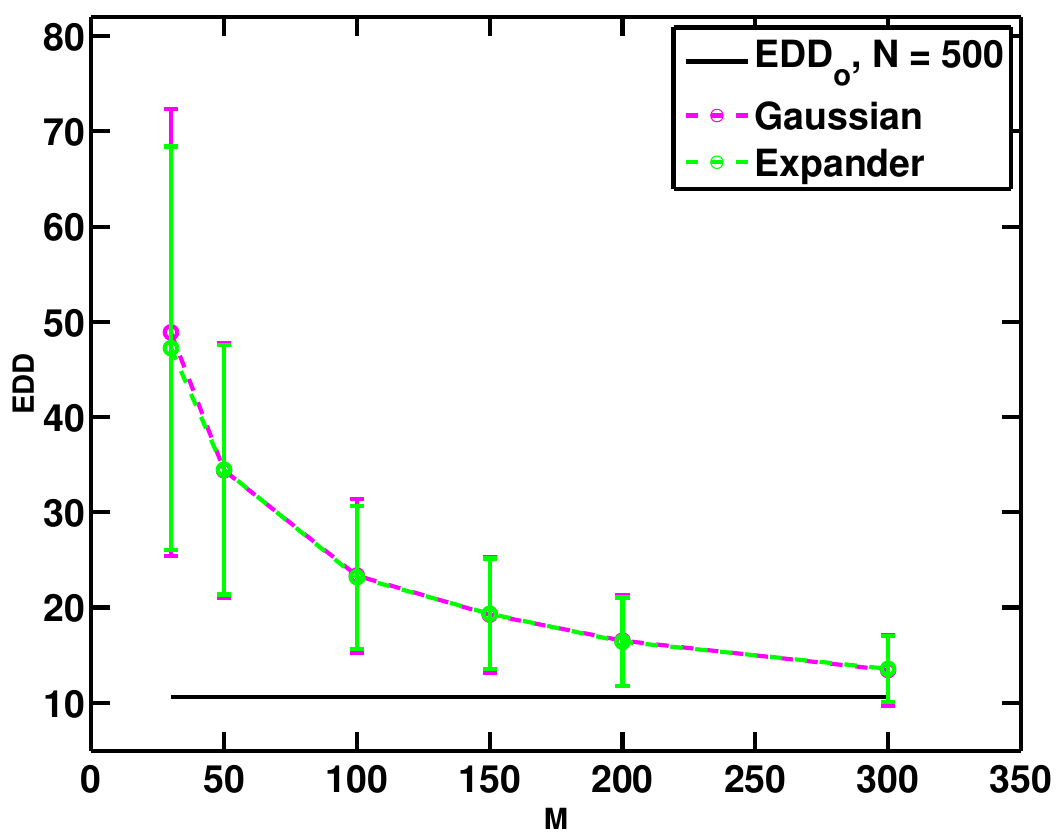}
\end{tabular}
\end{center}
\caption{\yc{Comparison of EDDs for $A$ being a Gaussian random matrix versus an expander graph when $[\mu]_i$'s are i.i.d. generated from $[-3,3]$.}}
\label{fig:EDD_Twotype_M_posi_nega}
\end{figure}

\subsection{Time-varying projections with 0-1 matrices}
\label{sec:varyingnumerical}

To demonstrate the performance of the procedure $T_{\{0,1\}}$ (\ref{proc3}) using time-varying projection with 0-1 entries, again, we consider two cases: the post-change mean vector $[\mu]_i = \mu_0$ and the post-change mean vector has $100p\%$ entries $[\mu]_i$ being one and the other entries being zero.
The simulated EDDs are shown in Fig. \ref{fig:EDD_randomrow_p}. Note that $T_{\{0,1\}}$ can detect change quickly with a small subset of observations. Although EDDs of $T_{\{0,1\}}$ are larger than those for the fixed projections in Fig. \ref{fig:EDD_comp} and Fig. \ref{fig:EDD_expander}, this example shows that projection with 0-1 entries can have little performance loss in some cases, and it is still a viable candidate since such projection means a simpler measurement scheme.

\begin{figure}[h]
\begin{center}
\begin{tabular}{c}
\includegraphics[width=0.7\linewidth]{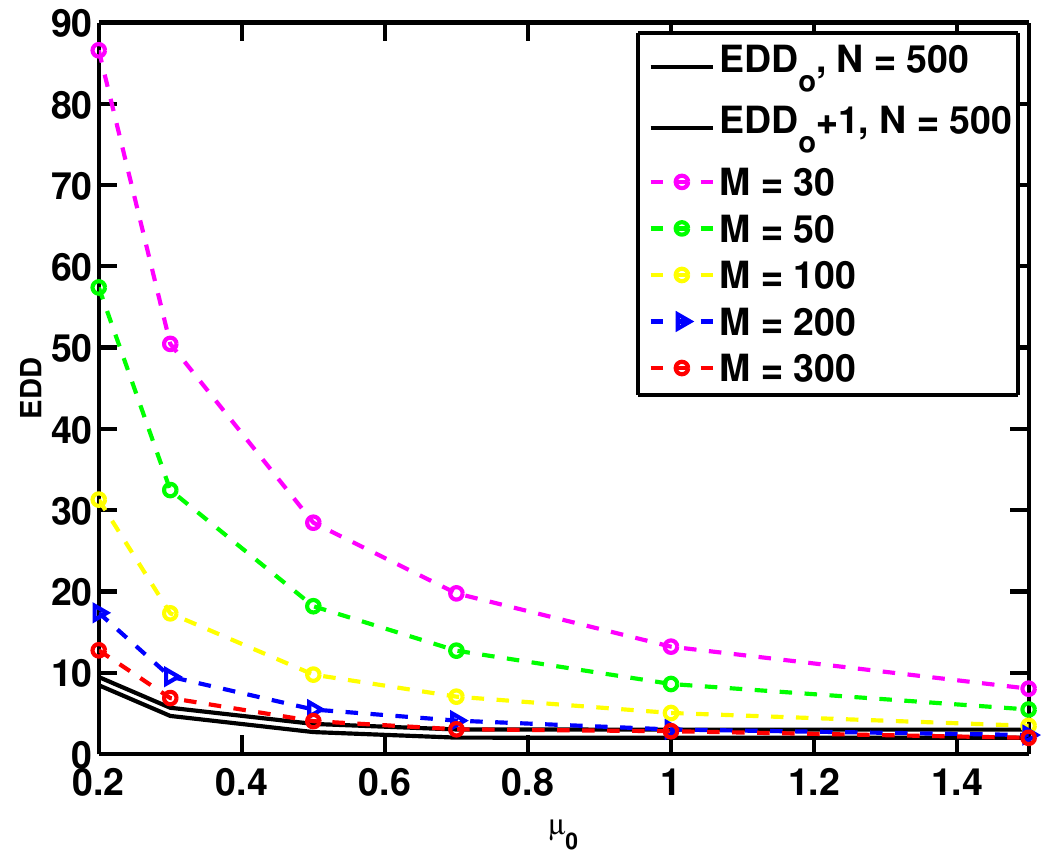}\\
(a)\\
\includegraphics[width=0.7\linewidth]{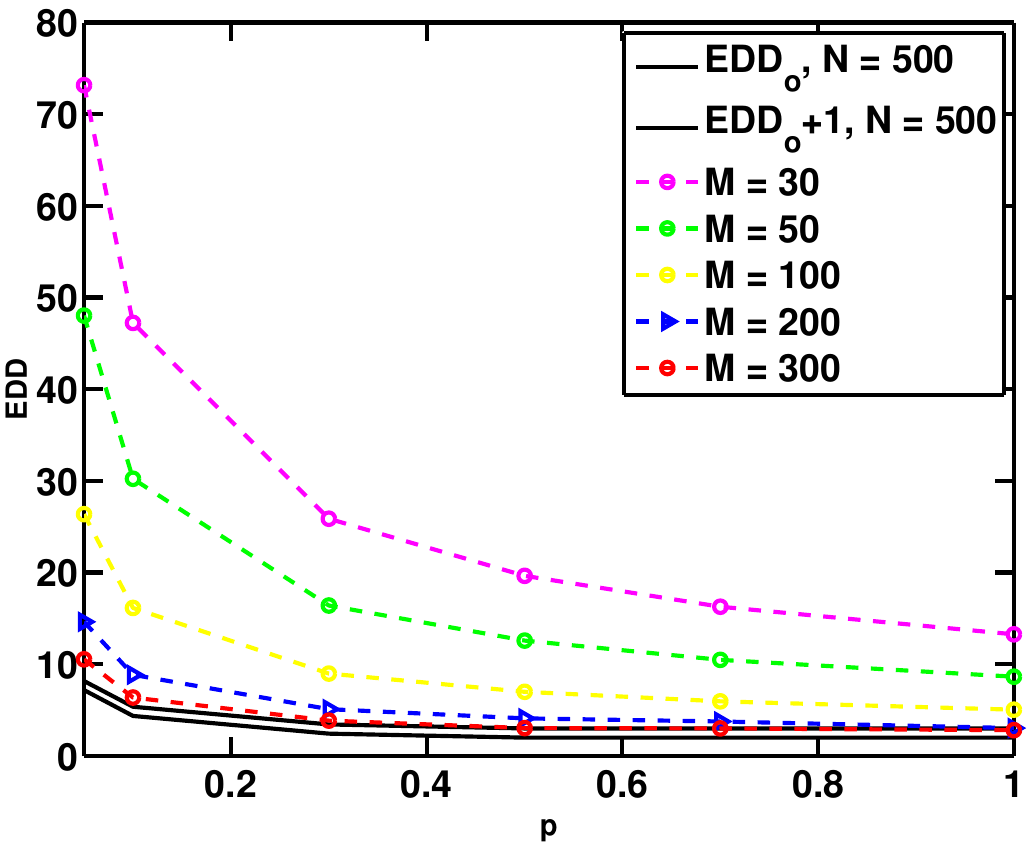}\\
(b)
\end{tabular}
\end{center}
\caption{$A_t'$s are time-varying projections.  \yc{The standard deviation of each point is less than half of its value.} (a) EDD versus $\mu_0$, when all $[\mu]_i = \mu_0$; (b) EDD versus $p$ when we randomly select $100p\%$ entries $[\mu]_i$ to be one and set the other entries to be zero; the smallest value of $p$ is 0.05.}
\label{fig:EDD_randomrow_p}
\end{figure}

\subsection{Comparison with multivariate CUSUM}

We compare our sketching method with a benchmark adapted from the conventional multivariate CUSUM procedure \cite{woodall1985multivariate} for the sketches. A main difference is that in multivariate CUSUM, one needs a prescribed post-change mean vector (which is set to be an all-one vector in our example), rather than estimate it as the GLR statistic does. Hence, its performance may be affected by parameter misspecification. We compare the performance again in two settings, when all $[\mu]_i$ are equal to a constant, and when $100p\%$ entries of the post-change mean vector are positive valued. In Fig. \ref{fig:EDD_compare_multiCUSUM}, the log-GLR based sketching procedure performs much better than the multivariate CUSUM.

  \begin{figure}[h]
\begin{center}
\begin{tabular}{c}
\includegraphics[width=0.7\linewidth]{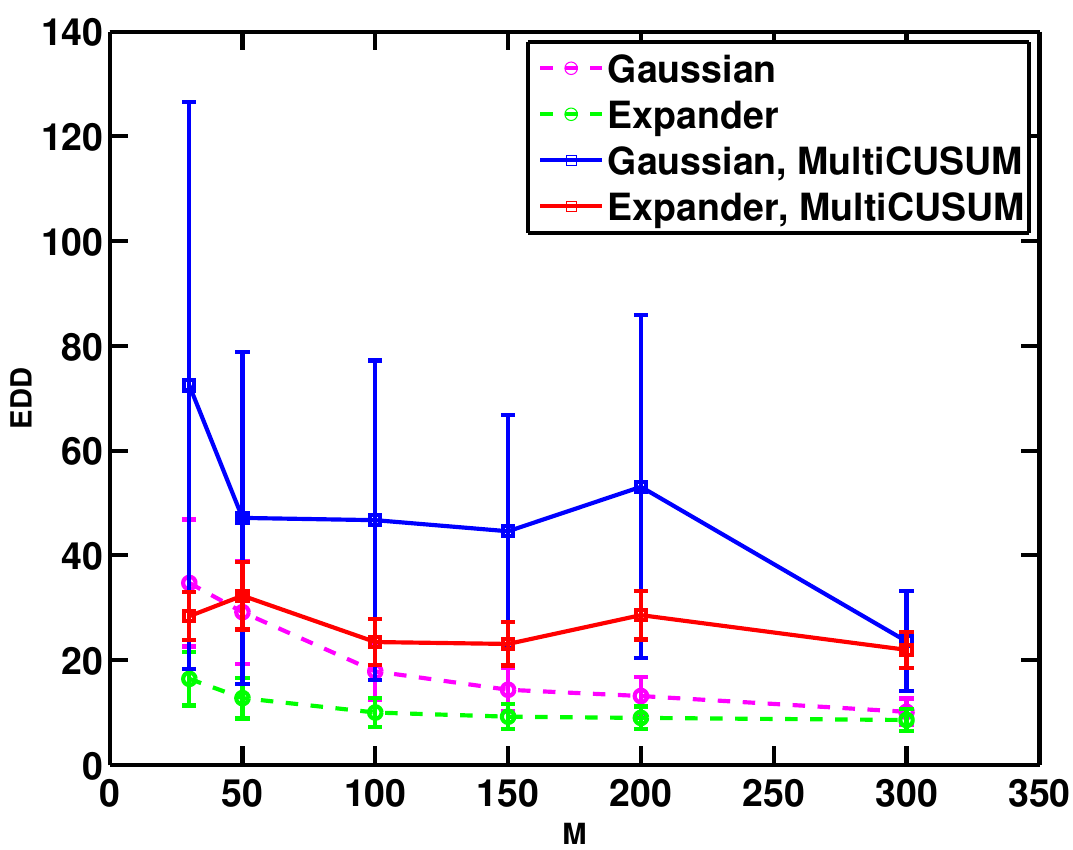}\\
(a)\\
\includegraphics[width=0.68\linewidth]{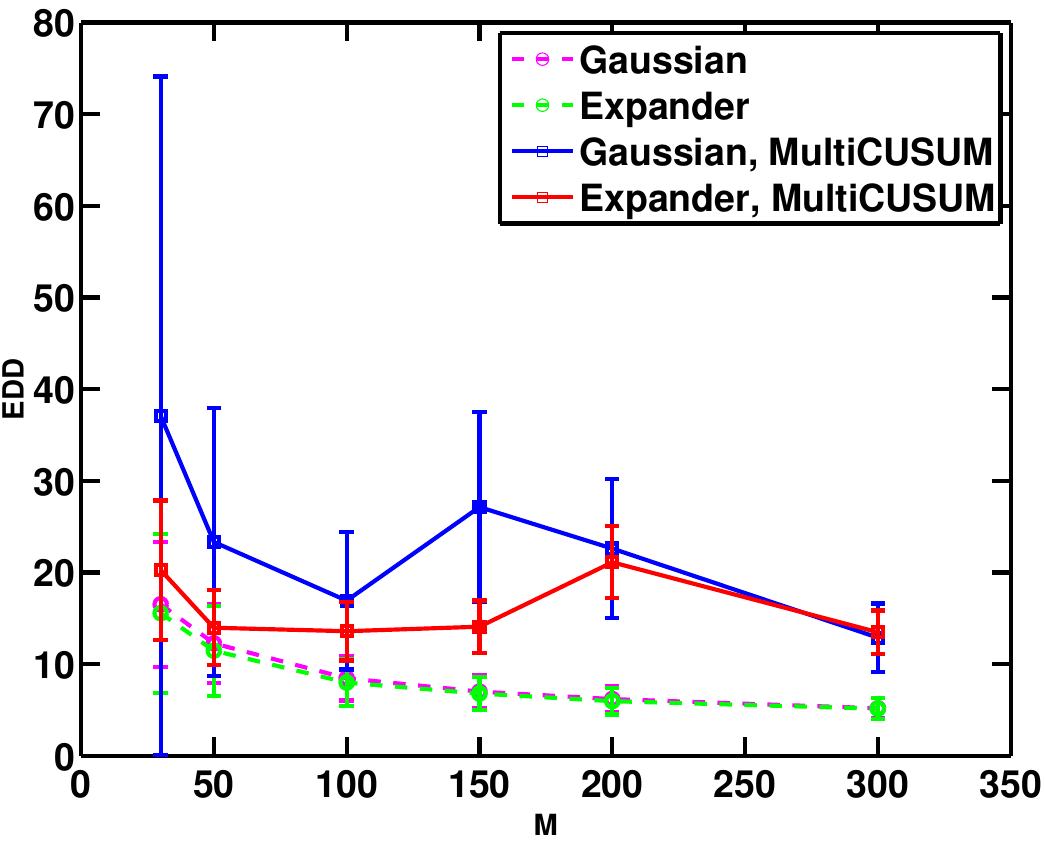}\\
(b)
\end{tabular}
\end{center}
\caption{\yc{Comparison of the sketching procedure with a method adapted from multivariate CUSUM. (a) EDDs versus various $M$s, when all $[\mu]_i = 0.2$; (b) EDDs versus various $M$s, when we randomly select $10\%$ entries $[\mu]_i$ to be one and set the other entries to be zero.}}
\label{fig:EDD_compare_multiCUSUM}
\end{figure}

\section{Examples for real applications}\label{sec:real}
\subsection{Solar flare detection}

We use our method to detect a solar flare in a video sequence from the Solar Data Observatory (SDO)\footnote{The video can be found at http://nislab.ee.duke.edu/MOUSSE/. The Solar Object Locator for the original data is SOL2011-04-30T21-45-49L061C108.}. Each frame is of size $232 \times 292$ pixels, which results in an ambient dimension $N=67744$. In this example, the normal frames are slowly drifting background  sun surfaces, and the anomaly is a much brighter transient solar flare emerges at $t=223$. %The true change-point $t=223$ is hand-picked. 
Fig. \ref{fig:solarflare}(a) is a snapshot of the original SDO data at $t=150$ before the solar flare emerges, and Fig. \ref{fig:solarflare}(b) is a snapshot at $t=223$ when the solar flare emerges as a brighter curve in the middle of the image. We  preprocess the data by tracking and removing the slowly changing background with the MOUSSE algorithm \cite{xie2013change} to obtain tracking residuals. The Gaussianity for the residuals, which corresponds to our $x_t$, is verified by the Kolmogorov-Smirnov test. For instance, the p-value is $0.47$ for the signal at $t=150$, which indicates that the Gaussianity is a reasonable assumption.

We apply the sketching procedure with fixed projection to the MOUSSE residuals. Choosing the sketching matrix $A$ to be an $M$-by-$N$ Gaussian random matrix with entries i.i.d. $\mathcal{N}(0,1/N)$. Note that the signal is deterministic in this case. To evaluate our method, we run the procedure $500$ times, each time using a different random Gaussian matrix as the fixed projection $A$. Fig. \ref{fig:EDD_solarflare} shows the error-bars of the EDDs from $500$ runs. As $M$ increases, both the means and standard deviations of the EDDs decrease. When $M$ is larger than $750$,  EDD is often less than $3$, which means that our sketching detection procedure can reliably detect the solar flare with only $750$ sketches. This is a significant reduction and the dimensionality reduction ratio is $750/67744 \approx  0.01$.

\begin{figure}[h!]
\begin{center}
\begin{tabular} {cc}
\includegraphics[width=4cm,height=3cm]{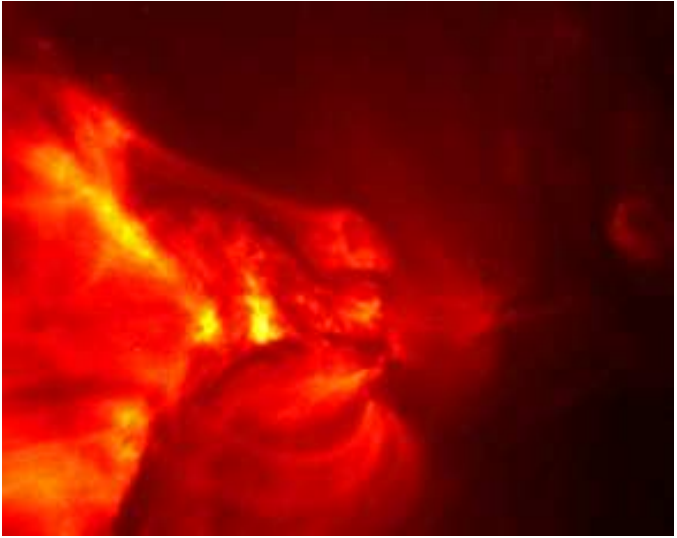} &
\includegraphics[width=4cm,height=3cm]{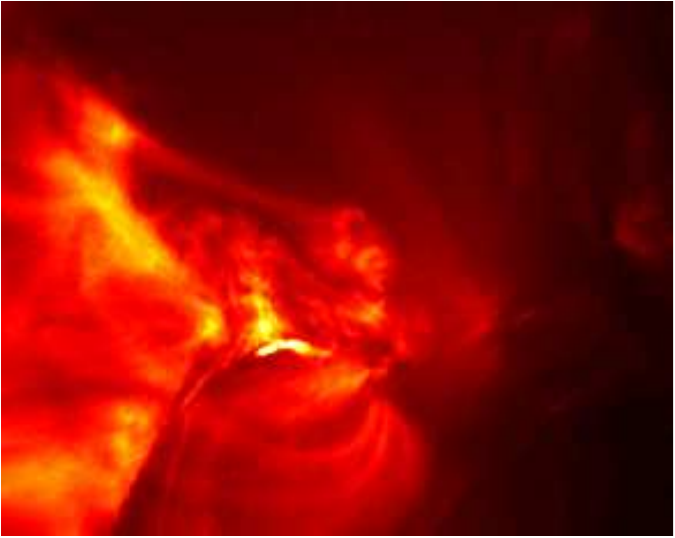} \\
(a) & (b) \\
% \includegraphics[width=4cm,height=3cm]{mousse_noflare.pdf} &
 %\includegraphics[width=4cm,height=3cm]{mousse_flare.pdf}\\
%(c) & (d) \\
\end{tabular}
\end{center}
\caption{Snapshot of the original solar flare data (a) at $t=150$; (b) at $t=223$. The true change-point location is at $t=223$.  
}
%\label{fig:solar}
\label{fig:solarflare}
\end{figure}

\begin{figure}[h]
\begin{center}
\begin{tabular}{c}
\includegraphics[width=0.7\linewidth]{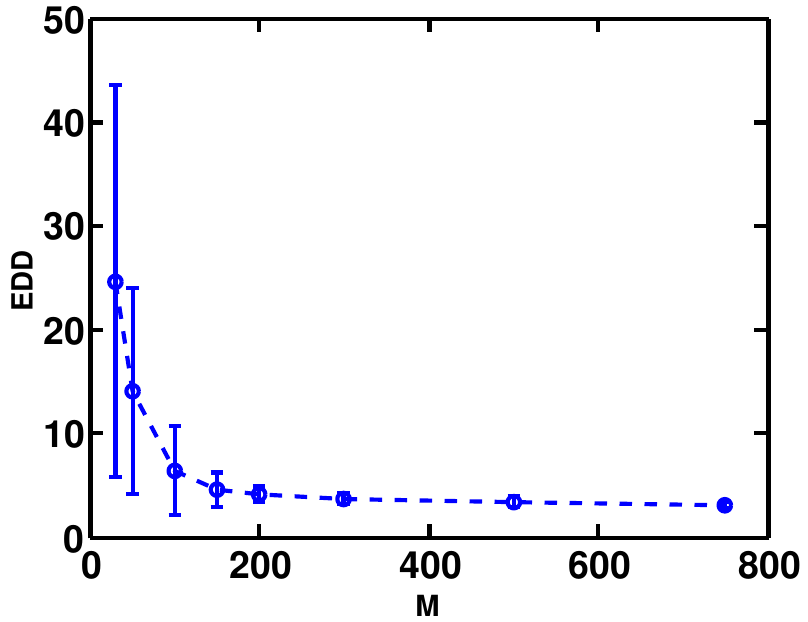}
\end{tabular}
\end{center}
\caption{Solar flare detection: EDD versus various $M$ when $A$ is an $M$-by-$N$ Gaussian random matrix. The error-bars are obtained from $10^4$ repetitions with runs with different Gaussian random matrix $A$.}
\label{fig:EDD_solarflare}
\end{figure}

\subsection{Change-point detection for power systems}

Finally, we present a synthetic example based on the real power network topology. We consider the Western States Power Grid of the United States, which consists of $4941$ nodes and $6594$ edges.  The minimum degree of a node in the network is $1$, as shown in Fig. \ref{fig:topology}\footnote{The topology of the power network can be downloaded at http://networkdata.ics.uci.edu/data/power/ \cite{watts1998collective}.}. The nodes represent generators, transformers, and substations, and edges represent high-voltage transmission lines between them \cite{watts1998collective}.
Note that the graph is sparse and that there are many ``communities'' which correspond to densely connected subnetworks. 

In this example, we simulate a situation for power failure over this large network. \yc{Assume that at each time we may observe the real power injection at an edge. When the power system is in a steady state, the observation is the true state plus Gaussian observation noise \cite{PowerSystemBook04}. We may estimate the true state (e.g., using techniques in \cite{PowerSystemBook04}), subtract it from the observation vector, and treat the residual vector as our signal $x_i$, which can be assumed to be i.i.d. standard Gaussian.} When a failure happens in a power system, there will be a shift in the mean for a small number of affected edges, since in practice, when there is a power failure, usually only a small part of the network is affected simultaneously.

To perform sketching, at each time, we randomly choose $M$ nodes in the network and measure the sum of the quantities over all attached edges as shown in Fig. \ref{fig:powerguide}. This corresponds to $A_t'$s with $N=6594$ and various $M<N$. Note that in this case, our projection matrix is a 0-1 matrix whose structure is constrained by the network topology. Our example is a simplified model for power networks and aims to shed some light on the potential of our method applied to monitoring real power networks.

\begin{figure}[h]
\begin{center}
\begin{tabular}{c}
\includegraphics[width=0.7\linewidth]{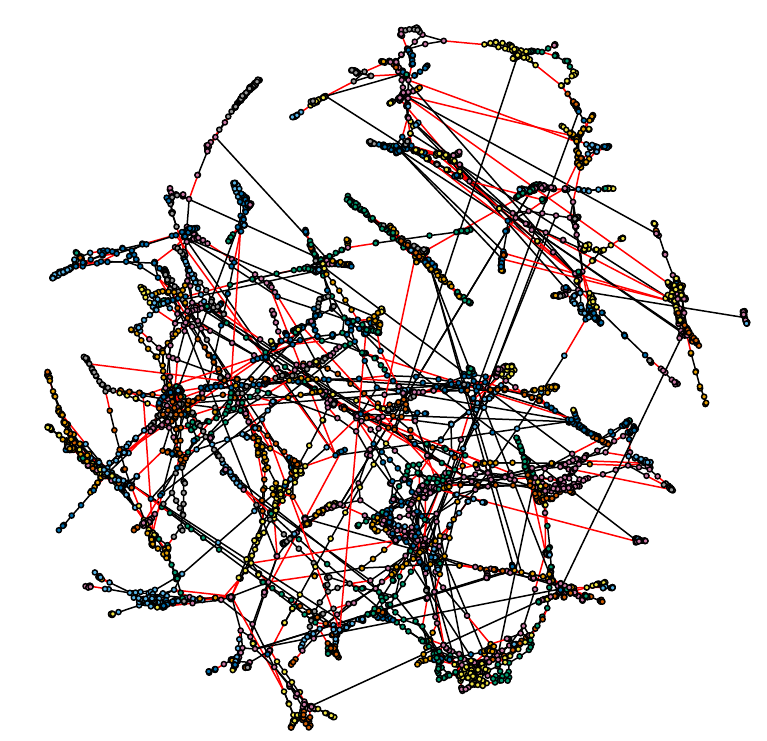}
\end{tabular}
\end{center}
\caption{Power network topology of the Western States Power Grid of the United States.}
\label{fig:topology}
\end{figure}

In the following experiment, we assume that on average $5\%$ of the edges in the network increase by $\mu_0$. Set the threshold $b$ such that the ARL is $5000$. Fig. \ref{fig:EDD_topology} shows the simulated EDD versus an increasing signal strength $\mu_0$. Note that the EDD from using a small number of sketches is quite small if $\mu_0$ is sufficiently large. For example, when $\mu_0=4$, one may detect the change by observing from only $M=100$ sketches (when the EDD is increased only by one sample), which is a significant dimesionality reduction with a ratio of $100/4941 \approx 0.02$.

\begin{figure}[h]
%\vspace{-0.2in}
\begin{center}
\begin{tabular}{c}
\includegraphics[width=0.7\linewidth]{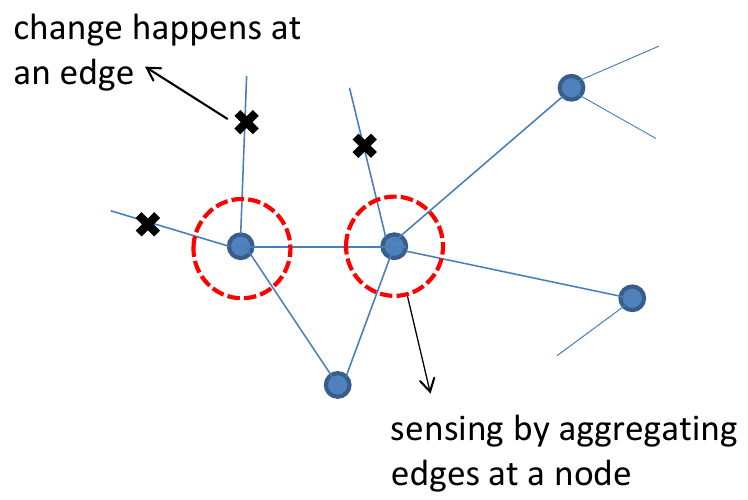}
\end{tabular}
\end{center}
\caption{Illustration of the measurement scheme for a power network. Suppose the physical quantities at edges (e.g., real power flow) at time $i$ form the vector $x_i$, and we can observe the sum of the edge quantities at each node. When there is a power failure, some edges are affected, and their means are shifted.}
\label{fig:powerguide}
\end{figure}

\begin{figure}[h]
\begin{center}
\begin{tabular}{c}
\includegraphics[width=0.7\linewidth]{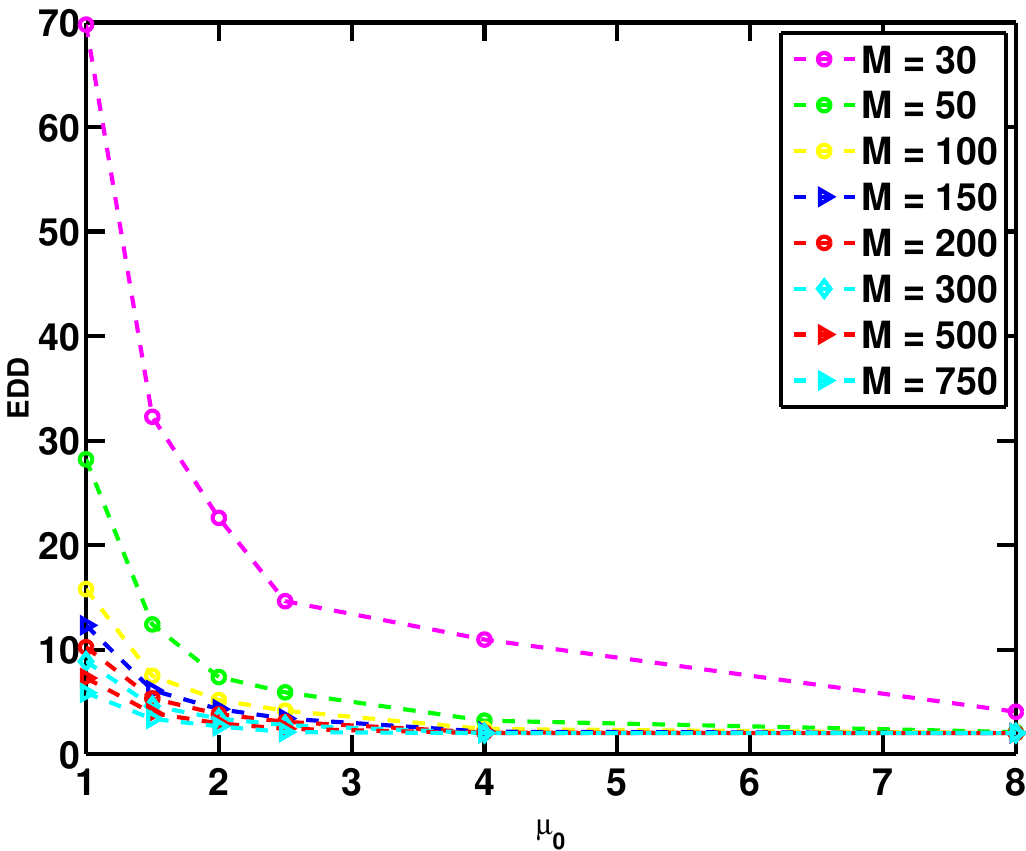}
%\\
%(a)\\
%\includegraphics[width=0.7\linewidth]{New_EDD_mesh_A_topology_changeP.pdf}\\
%(b)
\end{tabular}
\end{center}
\caption{Power system example: $A$ being a power network topology constrained sensing matrix. \yc{The standard deviation of each point is less than half of the value.} EDD versus $\mu_0$ when we randomly select $5\%$ edges with mean shift $\mu_0$. 
}
\label{fig:EDD_topology}
\end{figure}

\section{Summary and Discussions}
\label{sec:summary}

In this paper, we studied the problem of sequential change-point detection when the observations are linear projections of the high-dimensional signals. The change-point causes an {\it unknown} shift in the mean of the signal, and one would like to detect such a change as quickly as possible. We presented new sketching procedures for fixed and time-varying projections, respectively. \yang{Sketching is used to reduce the dimensionality of the signal and thus computational complexity; it also reduces data collection and transmission burdens for large systems.} 

The sketching procedures were derived based on the generalized likelihood ratio statistic. We analyzed the theoretical performance of our procedures by deriving approximations to the average run length (ARL) when there is no change, and the expected detection delay (EDD) when there is a change. Our approximations were shown to be highly accurate numerically and were used to understand the effect of sketching. 

We also characterized the \yang{relative performance of the sketching procedure compared to that without sketching.} We specifically studied the relative performance measure for fixed Gaussian random projections and expander graph projections. Our analysis and numerical examples showed that the performance loss due to sketching could be quite small \yang{in a big regime when the signal strength and the dimension of sketches $M$ are sufficiently large. Our result can also be used to find the minimum required $M$ given a  worst-case signal and a target ARL. In other words, we can  determine the region where sketching results in little performance loss.} We demonstrate the good performance of our procedure using numerical simulations and two real-world examples for solar flare detection and failure detection in power networks.

\yang{On a high level, although after sketching, the Kullback-Leibler (K-L) divergence becomes smaller, the {\it threshold} $b$ for the same ARL  also becomes smaller. For instance, for Gaussian matrix, the reduction in K-L divergence is compensated by the reduction of the threshold $b$ for the same ARL, because the factor that they are reduced by are roughly the same. This leads to the somewhat counter-intuitive result that the EDDs with and without sketching turns to be similar in this big regime.}

Thus far, we have assumed that the data streams are independent. In practice, if the data streams are dependent on a {\it known} covariance matrix $\Sigma$, we can whiten the data streams by applying a linear transformation $\Sigma^{-1/2}x_t$. Otherwise, the covariance matrix $\Sigma$ can also be estimated using a training stage via regularized maximum likelihood methods (see \cite{FanLiaoLiu2016} for an overview). Alternatively, we may estimate the covariance matrix $\Sigma'$ of the sketches $A\Sigma A\transpose$ or $A_t\Sigma A_t\transpose$ directly, which requires fewer samples to estimate due to the lower dimensionality of the covariance matrix. Then we can build statistical change-point detection procedures using $\Sigma'$ (similar to what has been done for the projection Hotelling control chart in \cite{change-point-bookchapt15}), which we leave for future work.

\section*{Acknowledgement}

%Thank Qingbin Li for running partial results.
This work is partially supported by NSF grants CCF-1442635 and CMMI-1538746, and an NSF CAREER Award CCF-1650913.
%\yc{1212}
%\clearpage

\bibliography{RD_proc}

\clearpage

\appendices

\section{Proofs}

We start by deriving the ARL and EDD for the sketching procedure.
\begin{proof}[Proofs for Theorem \ref{thmARL} and Theorem \ref{thmEDD}]
This analysis demonstrates that the sketching procedure corresponds to the so-called mixture procedure (cf. $T_2$ in \cite{XieSiegmund2012}) in a special case of $p_0 = 1$, $M$ sensors, and the post-change mean vector is $V\transpose \mu$. In \cite{XieSiegmund2012}, Theorem 1, it was shown that the ARL of the mixture procedure with parameter $p_0\in[0, 1]$ and $M$ sensors is given by
\begin{equation}
\mathbb{E}^\infty\{T\} \sim H(M, \theta_0)/ \underbrace{\int_{[2M\gamma(\theta_0)/m_1]^{1/2}}^{[2M\gamma(\theta_0)/m_0]^{1/2}}  y \nu^2(y) dy}_{c'(M, b, w)}, \label{ET1}
\end{equation}
where the detection statistic will search within a time window $m_0\leq t-k\leq m_1$. Let $g(x,p_0) = \log(1-p_0 + p_0e^{x^2/2})$. Then $\psi(\theta) = \log \mathbb{E}\{e^{\theta g(U, p_0)}\}$ is the log moment generating function (MGF) for $g(U, p_0)$, $U\sim \mathcal{N}(0, 1)$,
 $\theta_0$ is the solution to $\dot{\psi}(\theta) = b/M$,
\begin{equation} \label{abbrev}
H(M,\theta) = \frac {\theta
[2\pi  \ddot{\psi}(\theta)]^{1/2}}{ \gamma(\theta) M^{1/2}}
\exp\{M[\theta \dot{\psi}(\theta) - \psi(\theta)]\},
\end{equation}
and
\[
\gamma(\theta) = \frac{1}{2}\theta^2
\mathbb{E}\{[\dot{g}(U, p_0)]^2 \exp[\theta g(U, p_0) - \psi(\theta)]\}.
\]%When $p_0=1$, $\phi(\theta)$ is the log moment generating function (MGF) of $U^2/2$, where $U\sim \mathcal{N}(0, 1)$.
Note that $U^2$ is $\chi^2_1$ distributed, whose MGF is given by $\mathbb{E}\{\e^{\theta U^2}\}=1/\sqrt{1-2\theta}$. Hence, when $p_0 = 1$,
\[
\psi(\theta) = \log\mathbb{E}\{e^{\theta U^2/2}\}
%= \log \{(1-2\frac{\theta}{2})^{-1/2}\}
 = -\frac{1}{2}\log(1-\theta).
\]
The first-order and second-order derivative of the log MGF are given by, respectively,
\ben
\dot{\psi}(\theta) = \frac{1}{2(1-\theta)}, \quad
\ddot{\psi}(\theta) =  \frac{1}{2(1-\theta)^2}
\een
Set $\dot{\psi}(\theta_0) = b/M$. We obtain the solution that $1-\theta_0 = M/(2b)$, and $\theta_0 = 1-M/(2b)$. Hence, $\ddot{\phi}(\theta_0) = 2b^2/M^2$.
We have $g(x,1) = x^2/2$, and $\dot{g}(x, 1) = x$.
\begin{align*}
\gamma(\theta) %&= \frac{1}{2}\theta^2 \mathbb{E}\{[\dot{g}(U)]^2 \exp[\theta g(U) - \psi(\theta)]\} \\
& = \frac{\theta^2}{2} \mathbb{E}\{U^2 e^{\frac{\theta U^2}{2}}\} e^{\log\sqrt{1-\theta}} \\
&= \frac{\theta^2}{2} \cdot \frac{1}{(1-\theta)^{3/2}} \cdot \sqrt{1-\theta} = \frac{\theta^2}{2(1-\theta)},
\end{align*}
where
\begin{align*}
\mathbb{E}\{U^2 e^{\frac{\theta U^2}{2}}\}
=& \frac{1}{\sqrt{2\pi}} \int x^2 e^{\frac{\theta x^2}{2}} e^{-\frac{x^2}{2}} dx \\
=& \frac{1}{\sqrt{2\pi}} \int x^2 e^{-\frac{x^2}{2/(1-\theta)}} dx = \frac{1}{(1-\theta)^{3/2}}.
\end{align*}
Combining the above, we have that the ARL of the sketching procedure is given by
\begin{equation}
\begin{split}
\mathbb{E}^\infty \{T\}
%&= \frac{\theta_0 [2\pi \ddot{\psi}(\theta_0)]^{1/2}}{c(N) \gamma(\theta_0)\sqrt{N}} e^{N[\theta_0 \dot{\psi}(\theta_0) - \psi(\theta_0)]} \\
&= \frac{\theta_0 [2\pi\cdot \frac{1}{2(1-\theta_0)^2}]^{1/2}}
{c'(M, b, w)\frac{\theta_0^2}{2(1-\theta_0)}\sqrt{M}} e^{\frac{M\theta_0}{2(1-\theta_0)}}(1-\theta_0)^{M/2} + o(1)\\
&= \frac{\sqrt{\pi}}{c'(M, b, w)}\frac{2}{\theta_0\sqrt{M}} e^{\frac{M\theta_0}{2(1-\theta_0)}}(1-\theta_0)^{M/2} + o(1).
\end{split}
\label{ET_1}
\end{equation}
Next, using the fact that $1/(1-\theta_0) = 2b/M$, we have that the two terms in the above expression can be written as
\[
\frac{M\theta_0}{2(1-\theta_0)} = \frac{M\theta_0}{2} \frac{2b}{M} = \theta_0 b, \quad (1-\theta_0) = \frac{M}{2b},
\]
then (\ref{ET_1}) becomes
\begin{align*}
&\mathbb{E}^\infty\{T\}=\frac{\sqrt{\pi}}{c'(M, b, w)} \frac{2}{\theta_0 \sqrt{M}} e^{\theta_0 b} (\frac{M}{2b}
)^{\frac{M}{2}} + o(1)\\
=& \frac{2\sqrt{\pi}}{c'(M, b, w)} \frac{1}{\sqrt{M}} \frac{1}{1-\frac{M}{2b}} e^{(b-\frac{M}{2})}  (\frac{M}{2b}
)^{\frac{M}{2}} + o(1)
\end{align*}
\begin{align*}
=& \frac{2\sqrt{\pi}}{c'(M, b, w)} \frac{1}{\sqrt{M}} \frac{1}{1-\frac{M}{2b}}
(\frac{M}{2})^{\frac{M}{2}} b^{-\frac{M}{2}} e^{b-\frac{M}{2}} + o(1).
\end{align*}
Finally, note that we can also write
\[\gamma(\theta_0) = \theta_0^2/[2(1-\theta_0)] = (1-M/(2b))^2/(M/b),\] and the constant is
\begin{equation}
\begin{split}
c'(M, b, w) =& \int_{[2M\gamma(\theta_0)/w]^{1/2}}^{[2M\gamma(\theta_0)]^{1/2}}  y \nu^2(y) dy \\
=& \int_{\sqrt{\frac{2b}{w}}(1-\frac{M}{2b})}^{\sqrt{2b}(1-\frac{M}{2b})} y \nu^2(y) dy.
\end{split}
\end{equation}
We are done deriving the ARL. The EDD can be derived by applying Theorem 2 of \cite{XieSiegmund2012} in the case where $\Delta = \|V\transpose\mu\|$, the number of sensors is $M$, and $p_0 = 1$.
\end{proof}

The following proof is for the Gaussian random matrix $A$.

\begin{proof}[Proof of Theorem \ref{mean_concentration}]
It follows from (\ref{quotient_result}), and a standard result concerning the distribution function of the beta distribution~\cite[26.5.3]{handbook}, that
\begin{equation}\label{dist_fn}
\mathbb{P}\{\Gamma\le b\}=I_b\left(\frac{M}{2}, \frac{N-M}{2}\right),
\end{equation}
where $I$ is the regularized incomplete beta function (RIBF)~\cite[6.6.2]{handbook}. We first prove the lower bound in (\ref{concentration_result}). Assuming $N\rightarrow\infty$ such that (\ref{delta_def}) holds, we may combine (\ref{dist_fn}) with \cite[Theorem 4.18]{mythesis} to obtain
\begin{equation*}
\begin{split}
&\displaystyle\lim_{(M, N)\rightarrow \infty}\frac{1}{N} \ln\mathbb{P}\{\Gamma\le\delta-\epsilon\}\\
&= - \left[\delta \ln \left(\frac{\delta}{\delta-\epsilon}\right) + (1-\delta)\ln\left(\frac{1-\delta}{1-\delta + \epsilon}\right)\right] = -c < 0,
\end{split}
\end{equation*}
from which it follows that there exists $\tilde{N}$ such that, for all $N\geq\tilde{N}$,
$$\frac{1}{N} \ln\mathbb{P}\{\Gamma\le\delta-\epsilon\}<-\frac{c'}{2},$$
which rearranges to give
$$\mathbb{P}\{\Gamma\le\delta-\epsilon\} <e^{\frac{-c'N}{2}},$$
which proves the lower bound in (\ref{concentration_result}). To prove the upper bound, it follows from (\ref{dist_fn}) and a standard property of the RIBF~\cite[6.6.3]{handbook} that
\begin{equation}\label{comp_dist_fn}
\mathbb{P}\{\Gamma\geq b\}=I_{1-b}\left(\frac{N-M}{2}, \frac{M}{2}\right).
\end{equation}
Assuming $N\rightarrow\infty$ such that (\ref{delta_def}) holds, we may combine (\ref{comp_dist_fn}) with \cite[Theorem 4.18]{mythesis} to obtain
\begin{equation*}
\begin{split}
&\displaystyle\lim_{(M, N)\rightarrow \infty}\frac{1}{N} \ln\mathbb{P}\{\Gamma\geq\delta+\epsilon\}\\
&= - \left[(1-\delta)\ln \left(\frac{1-\delta}{1-\delta-\epsilon}\right) + \delta\ln\left(\frac{\delta}{\delta + \epsilon}\right)\right] = -d < 0,
\end{split}
\end{equation*}
and the argument now proceeds analogously to that for the lower bound.
\end{proof}

\begin{lemma}\label{lem:2norm}
If a 0-1 matrix $A$ has constant column sum $d$, %corresponds to a $(s, \epsilon)$ expander with regular right degree $d$,
for every non-negative vector $x$ such that $[x]_i \geq 0$,  we have %(Meng: only hold for nonnegative signals)
%When $A$ is an $(k, \epsilon)$ expander, when $\bm \mu \geq \bf 0$, we have
\begin{equation}\label{eqn:lb}
\|A x\|_2 \geq \sqrt{d}\|x \|_2.
\end{equation}
\end{lemma}
\begin{proof}[Proof of Lemma \ref{lem:2norm}]
Below, $A_{ij} = [A]_{ij}$.
\begin{equation*}
\begin{split}
\|A x \|_2^2 & =  \sum_{i=1}^M \left(\sum_{j=1}^N A_{ij}x_j\right)^2 \\  & \geq   \sum_{i=1}^M \sum_{j=1}^N (A_{ij}x_j)^2  = d \|x \|_2^2.
\end{split}
              \end{equation*}
\end{proof}

\begin{lemma}[Bounding $\sigma_{\max}(A)$]\label{lem:sigma}
If $A$ corresponds to a $(s, \epsilon)$-expander with regular degree $d$ and regular left degree $c$, for any nonnegative vector $x$,
\begin{equation}
\frac{\|A x\|_2}{\|x\|_2} \leq d\sqrt{\frac{N}{M}},
\end{equation}
thus,
\begin{equation}\label{eqn:sigmamax}
\sigma_{\max}(A) \leq d \sqrt\frac{N}{M}.
\end{equation}
\end{lemma}

\begin{proof}[Proof of Lemma \ref{lem:sigma}]
For any nonnegative vector $x$,
\begin{eqnarray}
\|Ax\|_2^2 & = & \sum_{i=1}^M (\sum_{j=1}^N A_{ij}x_j)^2 \nonumber \\
              & = & \sum_{i=1}^M \big(\sum_{j=1}^N (A_{ij}x_j)^2 + \sum_{j=1}^N \sum_{l=1, l \leq j}^N (A_{ij}A_{il}x_jx_l) \big)  \nonumber \\
               & \leq & \sum_{i=1}^M \big(\sum_{j=1}^N (A_{ij}x_j)^2 + \sum_{j=1}^N \sum_{l=1, l \leq j}^N \frac{A_{ij}A_{il}}{2}(x_j^2+x_l^2) \big) \nonumber 
                                             \end{eqnarray}
                              \begin{eqnarray}
               & = & \sum_{i=1}^M \sum_{j=1}^N \sum_{l=1}^N \frac{A_{ij}A_{il}}{2}(x_j^2+x_l^2)   \nonumber\\
                              & = &  \sum_{j=1}^N \sum_{l=1}^N \sum_{i=1}^M \frac{A_{ij}A_{il}}{2}(x_j^2+x_l^2)   \nonumber \\
                             % & =&\sum_{j=1}^p  \sum_{i=1}^n (A_{ij}x_j)^2 + \sum_{j=1}^p \sum_{l=1, l \leq j}^p  \sum_{i=1}^n \frac{A_{ij}A_{il}}{2}(x_j^2+x_l^2)  \\
                             % & =&\sum_{j=1}^p  d(x_j)^2 + \sum_{j=1}^p \sum_{l=1, l \leq j}^p  \sum_{i=1}^n \frac{A_{ij}A_{il}}{2}(x_j^2+x_l^2)  \\ \label{eqn:right}
                              & \leq&\sum_{j=1}^N  dc(x_j)^2 %+ \sum_{j=1}^p \sum_{l=1, l \leq j}^p  \sum_{i=1}^n \frac{A_{ij}A_{il}}{2}(x_j^2+x_l^2)
                               \label{eqn:left}\\
                               & = & \frac{d^2N}{M}\|x\|_2^2. \label{eqn:2norm}
\end{eqnarray}
Above, (\ref{eqn:left}) holds since  %$A$ has exactly $d$ `1's, and all others are `0's for every column $j$. Since $A$ has $c$ `1's in each row,
for a given column $j$, $A_{ij}=1$ holds for exactly $d$ rows. And for each row $i$ of these $d$ rows, $A_{il}=1$ for exactly $c$ columns with $l \in \{1,\ldots,p\}$;  % for a given row $i$ and a given
(\ref{eqn:2norm}) holds since $dN=Mc$. Finally, from the definition of $\sigma_{\max}$, (\ref{eqn:sigmamax}) holds.
\end{proof}

\begin{proof}[Proof for Theorem \ref{thm:expander}]
Note that
\begin{eqnarray} \label{eqn:norm}
\Delta &= &(\mu \transpose VV\transpose \mu)^{1/2}
   =  (\mu\transpose A\transpose U \Sigma^{-2} U\transpose A  \mu)^{1/2}  \nonumber\\
   & \geq & \sigma_{\max}^{-1}(A)\|U\transpose A  \mu\|_2     =  \sigma_{\max}^{-1}(A)\| A  \mu\|_2, \label{eqn:Delta}
\end{eqnarray}
where $\sigma_{\max}=\sigma_{\max}(A)$, and (\ref{eqn:Delta}) holds since $U$ is a unitary matrix.
Thus, in order to bound $\Delta$, we need to characterize $\sigma_{\max}$, as well as $\| A \mu\|_2$  for a $s$ sparse vector $\mu$.
Combining (\ref{eqn:Delta}) with Lemma \ref{lem:2norm} and \ref{lem:sigma}, we have that for every nonnegative vector $\mu$, $[\mu]_i \geq 0$,
\begin{equation}\label{eqn:Deltabound}
\Delta \geq  \frac{1}{d} \sqrt\frac{M}{N}  \sqrt{d(1-\epsilon)}\| \mu \|_2= \sqrt{\frac{M(1-\epsilon)}{dN}}\|\mu \|_2.
\end{equation}
Finally, Lemma \ref{lem:goodexpander} characterizes the quantity $[M(1-\epsilon)/(dN)]^{1/2}$ in (\ref{eqn:Deltabound}) and establishes the existence of such an expander graph. When $A$ corresponds to an  $(\alpha N, \epsilon)$ expander described in Lemma \ref{lem:goodexpander}, $\Delta \geq  \|\beta  \mu \|_2$ for all non-negative signals $[\mu]_i \geq 0$ for some constant $\alpha$ and some constant $\beta = (\rho(1-\epsilon)/d)^{1/2}$. Done.
\end{proof}

\begin{proof}[Proof for Corollary \ref{M_over_b_range}]
We define that $x \triangleq M/b$, then Theorem \ref{thmARL} tells us that when $M$ goes to infinity, we have that 
\begin{equation}
\begin{split}
&\mathbb{E}^\infty\{T\} = \\
&\frac{2\sqrt{\pi}}{c(M, x, w) }  \frac{1}{1-\frac{x}{2}} \frac{1}{\sqrt{M}}
\left(\frac{x}{2}\right)^{\frac{M}{2}} \exp\left(\frac{M}{x}-\frac{M}{2}\right)+ o(1), \label{ET1}
\end{split}
\end{equation}
where
\begin{equation}
c(M, x, w)  = \int_{\sqrt{\frac{2M}{xw}}(1-\frac{x}{2})}^{\sqrt{\frac{2M}{x}}(1-\frac{x}{2})} u \nu^2(u) du, 
\end{equation}
and
\[
\nu(u) \approx \frac{2/u[\Phi(u/2) - 0.5]}{(u/2)\Phi(u/2) + \phi(u/2)}.
\]

Define that $\gamma \triangleq \mathbb{E}^{\infty}\{T\}$. One claim is that when $M>24.85$ and $\gamma \in [e^5, e^{20}]$ there exists one $x^* \in (0.5,2)$ such that (\ref{ET1}) holds. Next, we prove the claim. 

Define the logarithm of the right-hand side of (\ref{ET1}) as follows:
\begin{equation}
\begin{split}\nonumber
p(x) \triangleq& \log(2\sqrt{\pi}) - \log(C(M,x,w)) - \log \left(1-\frac{x}{2} \right) \\ 
&+\frac{M}{2}\log\frac{x}{2} + \frac{M}{x} - \frac{M}{2} -\frac{1}{2}\log M. 
\end{split}
\end{equation}

Since $\nu(u) \rightarrow 1$ as $u \rightarrow 0$ and $\nu(u) \rightarrow \frac{2}{u^2}$ as $u \rightarrow \infty$, we know that $\int_{0}^{\infty} u\nu^2(u)du$ exists. From the numerical integration, we know that $\int_0^{\infty} u\nu^2(u)du <1$. Therefore, $-\log (C(M,x,w)) >0$. Then,
\[
p\left(0.5\right) > \left( \frac{3}{2} - \frac{1}{2}\log 4 \right) M - \frac{1}{2}\log M + \log(2\sqrt{\pi}) - \log \frac{3}{4}. 
\]
When $M>24.85$, we have that $p(0.5) > 20$. Then, when $\gamma < e^{20}$ we have that $p(0.5)-\log \gamma >0$. 

Next, we prove that we can find some $x_0 \in (0.5, 2)$ such that $p(x_0) - \log \gamma <0$ provided that $\gamma > e^5$. 
Since $\phi \left( \frac{u}{2} \right) < 0.5$ and 
\[
0.5 + \frac{1}{\sqrt{2\pi}} \exp\left[ -\frac{1}{2} \left(\frac{u}{2}\right)^2 \right] \left( \frac{u}{2}\right)  \leq  \Phi \left( \frac{u}{2} \right) \leq 1,
\]
for any $u>0$. We have that 
\[
\nu(u) > \sqrt{ \frac{2}{\pi}} \cdot \frac{\exp \left( -\frac{u^2}{8}\right)}{u+1}. 
\]
Then, we have that for any $u>0$, 
\[
u\nu^2(u) > \frac{2}{\pi} \cdot \frac{u \cdot \exp(-u^2/4)}{(u+1)^2}. 
\]

We define that $x_0$ is the solution to the following equation:
\begin{equation}
\sqrt{ \frac{2M}{x} } \left( 1-\frac{x}{2} \right) = 1. 
\label{x0}
\end{equation}
Then, we have that 
\begin{equation} 
\begin{split} \nonumber
C(M,x_0,w) >& \frac{2}{\pi} \cdot \int_{1/\sqrt{w}}^{1}  \frac{u \cdot \exp(-u^2/4)}{(u+1)^2} du \\
>& \frac{2}{\pi} \cdot \int_{1/\sqrt{w}}^{1}  \frac{u \cdot \exp(-u^2/4)}{4} du  \\
=& \frac{1}{\pi} \cdot \left( \exp \left( -\frac{1}{4w} \right) - \exp \left( -\frac{1}{4} \right) \right) \\
>& \frac{1}{\pi}  \exp \left( -\frac{1}{4} \right) \cdot \left( \frac{1}{4} - \frac{1}{4w} \right),
\end{split}
\end{equation}
where the second inequality is due to the fact that the upper bound for the integral interval is $1$ and the third inequality is due to the fact that $\exp(-x)$ is a convex function. Therefore, we have that 
\[
-\log C(M,x_0,w) < \log \pi + \frac{1}{4} - \log \left( \frac{1}{4} - \frac{1}{4w} \right)
\]
Note that the upper bound above for $-\log C(M,x_0,w) $ is not dependent on $M$, which is because we choose a $x_0$ that depends on $M$. 
Solving the equation (\ref{x0}), we have that 
\[
x_0 = 2 + \frac{1}{M} - \sqrt{ \frac{1}{M^2} + \frac{4}{M}}<2, 
\]
and $x_0 \rightarrow 2$ as $M \rightarrow \infty$. 
By Taylor's expansion, we have that $x_0 = 2 - 2M^{-1/2} +M^{-1} + o(M^{-1})$, or $x_0 = 2 - 2M^{-1/2} + o(M^{-1/2})$. 
Then, we have that 
\[
-\log \left(1-\frac{x_0}{2} \right) = -\log (M^{-1/2} )+ o(1),
\]
and
\begin{equation}
\begin{split} \nonumber
\frac{M}{2}\log\frac{x_0}{2} =& \frac{M}{2} \log(1-M^{-1/2}) \\
=& \frac{M}{2}\cdot \left(-M^{-1/2}-\frac{1}{2}M^{-1}+o(M^{-1}) \right) \\
=& -\frac{1}{2} M^{1/2} - \frac{1}{4} + o(1),
\end{split}
\end{equation}
and
\begin{equation}
\begin{split}\nonumber
\frac{M}{x_0} =& \frac{M}{2} \cdot \frac{1}{1-(M^{-1/2}+M^{-1}/2+o(M^{-1}))}\\
=& \frac{M}{2} \cdot (M^{-1/2}+M^{-1}/2+o(M^{-1}) \\
&+ (M^{-1/2}+M^{-1}/2+o(M^{-1}))^2 + o(M^{-1})) \\
=& \frac{1}{2}M^{1/2} + \frac{1}{2} + o(1)
\end{split}
\end{equation}

Combining the above results, we have that 
\begin{equation}
p(x_0) < \log (2\sqrt{\pi}) + \log \pi - \log \left( \frac{1}{4} - \frac{1}{4w} \right) + \frac{1}{2} + o(1). 
\label{upperbound}
\end{equation}
One important observation is that the right-hand side of (\ref{upperbound}) converges as $M \rightarrow \infty$. In fact, $p(x_0)$  as a function of $M$ is decreasing and converges as $M \rightarrow \infty$. Since we set $w\geq 100$, then for any $M>24.85$, $p(x_0) < 5$. Therefore, for any $\gamma > e^5$ and any $M>24.85$, we can find a $x_0$ close to $2$ such that $p(x_0) - \log \gamma <0$. 

Since $p(x)$ is a continuous function, there exists a solution $x^* \in (0.5,2)$ such that equation (\ref{ET1}) holds. 

\end{proof}

\begin{proof}[Proof of Theorem \ref{thm_time_varying}]

The proof uses a similar argument as that in \cite{XieSiegmund2012}.

By law of large number, when $t-k$ tends to infinity, the following sum converges in probability 
\begin{equation}
\frac{1}{t-k}\sum_{i=k+1}^t \mathbb I_{in} \xrightarrow[]{p} r. \label{cen1}
\end{equation} 
Moreover, from central limit theorem, 
\begin{equation}
\frac{1}{\sqrt{t-k}}\sum_{i=k+1}^t [x_i]_n  (\mathbb I_{in}-r) \xrightarrow[]{d} \mathcal{N}(0, r(1-r)). 
\end{equation}
So by continuous mapping theorem, 
\begin{equation}
\left(\frac{1}{\sqrt{(t-k)r(1-r)}}\sum_{i=k+1}^t [x_i]_n  (\mathbb I_{in}-r) \right)^2 \xrightarrow[]{d} \chi^2_1,
\label{cen2}
\end{equation} i.e., the squared and scaled version of the sum is asymptotically a $\chi^2_1$ random variable with one degree of freedom. By Slutsky's theorem, combining (\ref{cen1}) and (\ref{cen2}), 
\[
\frac{1}{1-r}\frac{[\sum_{i=k+1}^t [x_i]_n  (\mathbb I_{in}-r) ]^2}{\sum_{i=k+1}^t \mathbb I_{in}} \xrightarrow[]{d} \chi^2_1
\]
Using Lemma 1 in \cite{LaurentMassart2000}, for $X\sim \chi^2_1$, 
\begin{align*}
&\mathbb{P}\{X \geq 1+2\sqrt{\epsilon} + 2\epsilon\} \leq e^{-\epsilon}\\
&\mathbb{P}\{X \leq 1-2\sqrt{\epsilon} \} \leq e^{-\epsilon}
\end{align*}
Therefore, with probability at least $1-2e^{-\epsilon}$, the difference is bounded by a constant
\[\left(\frac{\sum_{i=k+1}^t [x_i]_n  \mathbb I_{in}}{\sqrt{\sum_{i=k+1}^t \mathbb I_{in}}} - r\frac{\sum_{i=k+1}^t [x_i]_n }{\sqrt{\sum_{i=k+1}^t \mathbb I_{in}}}\right)^2 < (1+2\sqrt{\epsilon}+ 2\epsilon)(1-r). 
\]

On the the hand, by central limit theorem, when $t-k$ tends to infinity, 
\[
\frac{1}{\sqrt{t-k}} \sum_{i=k+1}^t [x_i]_n \xrightarrow[]{d} \mathcal{N}(0, 1).
\]
and by law of large number and continuous mapping theorem 
\[
\left(\frac{\sum_{i=k+1}^t \mathbb I_{in}}{t-k}\right)^{-1/2} - \frac{1}{\sqrt{r}} \xrightarrow[]{p}  0
\]
Hence, invoking Slutsky's theorem again, we have
\[
\left(
\frac{\sum_{i=k+1}^t [x_i]_n}{\sqrt{\sum_{i=k+1}^t \mathbb I_{in}}} - \frac{\sum_{i=k+1}^t [x_i]_n}{\sqrt{r(t-k)}}
\right)^2\xrightarrow[]{d} 0
%\frac{\sum_{i=k+1}^t [x_i]_n}{\sqrt{\sum_{i=k+1}^t \mathbb I_{in}}} - \frac{\sum_{i=k+1}^t [x_i]_n}{\sqrt{r(t-k)}}
\]
Hence, combining the above, by a triangle inequality type of argument, we may conclude that, with high-probability, the difference is bounded by a constant $c$
\[
\left(
\frac{\sum_{i=k+1}^t [x_i]_n  \mathbb I_{in}}{\sqrt{\sum_{i=k+1}^t \mathbb I_{in}}} - \sqrt{r}\frac{\sum_{i=k+1}^t [x_i]_n}{\sqrt{(t-k)}}\right)^2<c.\]

Hence, to control the ARL for the procedure defined in (\ref{proc3})
\begin{equation}
\begin{split}
T_{\{0,1\}}= &\\
 \inf\{t: &
\max_{t-w\leq k < t} \frac{1}{2}\sum_{n=1}^N \frac{\left(\sum_{i=k+1}^t [x_i]_n \mathbb I_{in}\right)^2}{\sum_{i=k+1}^t  \mathbb I_{in}} > b\},
\end{split}
\end{equation}
one can approximately consider another procedure
\[
\widetilde{T}_{\{0,1\}} = \inf \{t: \max_{t-w\leq k < t} \frac{1}{2}\sum_{n=1}^N U_{n,k,t}^2 >  \frac{b}{r}  \},
\] with 
\[
U_{n,k,t} \triangleq \frac{\sum_{i=k+1}^t [x_i]_n}{\sqrt{t-k}},
\]
This corresponds to the special case of the mixture procedure with $N$ sensors and all being affected by the change ($p_0 = 1$), except that the threshold is scaled by $1/r$. Hence, we can use the ARL approximation for mixture procedure, which leads to (\ref{ET2}). 

\end{proof}

\section{Justification for EDD of (\ref{DD2})}\label{justification}
\begin{proof}

Below, let $T=T'_{\{0,1\}}$ for simplicity. Define $S_{n,t} = \sum_{i=1}^t [x_i]_n \xi_{ni}$ for any $n$ and $t$.
To obtain an EDD approximation to $T'_{\{0,1\}}$, first we note that 
\begin{equation}
\begin{split}
&\frac{1}{2} \sum_{n=1}^N Z_{n,k,T}^2 
= \frac{1}{2} \sum_{n=1}^N \frac{\left(\sum_{i=k+1}^T [x_i]_n \xi_{ni}\right)^2}{r(T-k)} \\
=& \frac{1}{2} \sum_{n=1}^N \frac{\left(S_{n,T} - S_{n,k} \right)^2 }{r(T-k)} \\ 
=& \sum_{n=1}^N \mu_n\left[ S_{n,T} - S_{n,k} - (T-k) r\mu_n/2  \right] \\
+& \sum_{n=1}^N \left[ S_{n,T} - S_{n,k} - (T-k)r\mu_n \right]^2 / (2r(T-k)).
\end{split}
\end{equation}
Then we can leverage a similar proof as that to Theorem 2 in \cite{XieSiegmund2012} to obtain  that as $b \rightarrow \infty$,
\begin{equation}
\begin{split}
&\mathbb{E}^0\left\{\max_{0\leq k<T} \frac{1}{2} \sum_{n=1}^N Z_{n,k,T}^2\right\} \\
=& \mathbb{E}^0 \left\{\sum_{n=1}^N \mu_n(S_{n,T} - rT\mu_n/2) \right\} \\
&\quad  + \mathbb{E}^0 \left\{\sum_{n=1}^N (S_{n,T} - Tr\mu_n)^2/(2Tr) \right\} \\
&\quad + \mathbb{E}^0 \left\{ \min_{0\leq k<b^{1/2}} \left( \sum_{n=1}^N \mu_n (S_{n,k} - kr\mu_n/2) \right) \right\} + o(1).
\end{split}
\label{eq:threeterms}
\end{equation}
The first term on the right-hand side of (\ref{eq:threeterms}) is equal to $\mathbb{E}^0 \{T'_{\{0, 1\}}\} \cdot r\sum_{n=1}^N \mu_n^2 /2$. Using the fact that random variables $([x_i]_n \xi_{ni} - r\mu_n)/\sqrt{r}$ are i.i.d. with mean zero and unit variance, together with the Anscombe-Doeblin Lemma \cite{Siegmund1985}, we have that as $b \rightarrow \infty$, the second term on the right-hand side of (\ref{eq:threeterms}) is equal to $N/2+o(1)$. The third term can be shown to be small similar to \cite{XieSiegmund2012}. Finally, ignoring the overshoot of the detection statistic exceeding the detection threshold, we can replace the left-hand side of (\ref{eq:threeterms}) with $b$. Solving  the equation, we obtain the first order approximation of the EDD is given by (\ref{DD2}).
%\begin{equation}
%\mathbb{E}^0\{ T_{\{0,1\}}' \} = \left(\frac{2b-N}{\sum_{n=1}^N \mu_n^2} + o(1)\right)\cdot \frac{N}{M}.
%\label{DD2}
%\end{equation}

\end{proof}

\end{document}